\documentclass{article}

\usepackage{microtype}
\usepackage{graphicx}
\usepackage{booktabs} % for professional tables
\usepackage{wrapfig}

\usepackage{caption}
\usepackage{subcaption}
\usepackage{balance}
\usepackage{hyperref}

\usepackage{xcolor}% http://ctan.org/pkg/xcolor
\usepackage{bm}

\usepackage[accepted]{alt_icml}

\usepackage{booktabs}
\usepackage{amsfonts}
\usepackage{amsthm}
\usepackage{todonotes}
\usepackage{enumitem}
\usepackage{amssymb}
\usepackage{bigstrut}
\usepackage{mathtools}

\usepackage{makecell}

\newtheorem{theorem}{Theorem}
\newtheorem{prop}{Proposition}
\newtheorem{definition}{Definition}
\newtheorem{lemma}{Lemma}

\newtheorem{example}{Example}
\newtheorem*{question}{Question}
\newtheorem{remark}{Remark}

\numberwithin{equation}{section}

\DeclareMathOperator*{\sgn}{sign}
\DeclareMathOperator*{\conv}{conv}
\DeclareMathOperator*{\intr}{int}

\DeclareMathOperator*{\argmin}{argmin}

\DeclareMathOperator*{\vol}{vol}
\DeclareMathOperator*{\aug}{aug}

\newcommand{\PP}{\mathbb{P}}

\newcommand{\real}{\mathbb{R}}
\newcommand{\NN}{\mathbb{N}}

\newcommand{\mH}{\mathcal{H}}

\newcommand{\mR}{\mathcal{R}}

\newcommand{\mF}{\mathcal{F}}

\newcommand{\mB}{\mathcal{B}}

\newcommand{\mS}{\mathcal{S}}
\newcommand{\mA}{\mathcal{A}}

\newcommand{\ie}{\textit{i.e.}}

\begin{document}

\twocolumn[
\icmltitle{Does Data Augmentation Lead to Positive Margin?} 
\icmlsetsymbol{equal}{*}

\begin{icmlauthorlist}
\icmlauthor{Shashank Rajput}{equal,cs}
\icmlauthor{Zhili Feng}{equal,cs}
\icmlauthor{Zachary Charles}{ece}
\icmlauthor{Po-Ling Loh}{stats}
\icmlauthor{Dimitris Papailiopoulos}{ece}
\end{icmlauthorlist}

\icmlaffiliation{cs}{Department of Computer Science, University of Wisconsin-Madison}
\icmlaffiliation{ece}{Department of Electrical and Computer Engineering, University of Wisconsin-Madison}
\icmlaffiliation{stats}{Department of Statistics, University of Wisconsin-Madison}

\icmlcorrespondingauthor{Shashank Rajput}{rajput3@wisc.edu}
\icmlcorrespondingauthor{Zhili Feng}{zfeng49@cs.wisc.edu}
% \icmlkeywords{}

\vskip 0.3in
]



\printAffiliationsAndNotice{\icmlEqualContribution} % otherwise use the standard text.

\begin{abstract}
	
Data augmentation (DA) is commonly used during model training, as it significantly improves test error and model robustness.
DA artificially expands the training set by applying random noise, rotations, crops, or even adversarial perturbations to the input data.
Although DA is widely used, its capacity to provably improve robustness is not fully understood.
In this work, we analyze the robustness that DA begets by quantifying the \emph{margin} that DA enforces on empirical risk minimizers.
We first focus on linear separators, and then a class of nonlinear models whose labeling is constant within small convex hulls of data points.
We present lower bounds on the number of augmented data points required for non-zero margin, and show that commonly used DA techniques may only introduce significant margin after adding exponentially many points to the data set. 
\end{abstract}

% !TeX root = paper.tex

\section{Introduction}

	Modern machine learning has ushered in a plethora of advances in data science and engineering, which leverage models with millions of tunable parameters and achieve unprecedented accuracy on many vision, speech, and text prediction tasks.
	For state-of-the-art performance, model training involves stochastic gradient descent (SGD), combined with regularization, momentum, data augmentation, and other heuristics. 
	Several empirical studies \cite{zhang2016understanding,zantedeschi2017efficient} observe that among these methods, data augmentation plays a central role in improving the test error performance and robustness of these models.
 
	Data augmentation (DA) expands the training set with artificial data points.
 	For example, \citet{krizhevsky2012imagenet} augmented  ImageNet using translations, horizontal reflections, and altered intensities of the RGB channels of images in the training set. Others have augmented datasets by adding labels to sparsely annotated videos \cite{Misra_2015_CVPR, kuznetsova2015expanding, prest2012learning}. Another important class of data augmentation methods are referred to broadly as {\it adversarial training}. Such methods use adversarial examples \cite{szegedy2013intriguing,madry2017towards} to enlarge the training set. Many works have since shown that by training models on these adversarial examples, we can increase the robustness of learned models~\cite{bastani2016measuring, carlini2017towards,szegedy2013intriguing,DBLP:journals/corr/GoodfellowSS14}. Recently, \cite{ford2019adversarial} studied the use of additive Gaussian DA in ensuring robustness of learned classifiers. While they showed the approach can have some limited success, ensuring robustness to adversarial attacks requires augmenting the data set with Gaussian noise of particularly high variance. 
 	
 	The high-level motivation of DA is clear: a reliable model should be trained to predict the same class even if an image is slightly perturbed. Despite its empirical effectiveness, relatively few works theoretically analyze the performance and limitations of DA. \citet{bishop1995training} shows that training with noise is equivalent to Tikhonov regularization in expectation. \citet{wager2013dropout} show that training generalized linear models while randomly dropping features is approximately equivalent to $\ell_2$-regularization normalized by the inverse diagonal Fisher information matrix. \citet{dao2018kernel} study data augmentation as feature-averaging and variance regularization, using a Markov process to augment the dataset. \citet{wong2018provable} provide a provable defense against bounded $\ell_\infty$-attacks by training on a convex relaxation of the ``adversarial polytope,'' which is also a form of DA. 

	We take a different path by analyzing how DA impacts the margin of a classifier, \ie, the minimum distance from the training data to its decision boundary. We focus on margin since it acts as a proxy for both generalization \cite{shalev2014understanding} and worst-case robustness. In particular, we analyze how much data augmentation is necessary in order to ensure that any empirical risk minimization algorithm achieves positive, or even large, margin. To the best of our knowledge, no existing work has explicitly analyzed data augmentation through the lens of margin.

	\subsection{Contributions}\label{sec:problem_setup}
	
	We consider the following empirical risk minimization (ERM) problem:
	\begin{align*}
	\mA(S) = \argmin_{f\in\mF} \left\{\sum_{i=1}^n\ell(f(x_i),y_i)\right\}
	\end{align*}
	where $S = \{(x_i,y_i)\}_{i=1}^n$ is the training set, $x_i\in \mathbb{R}^d$ are the feature vectors, and $y_i\in \{-1,+1\}$ their labels. $\mF$ is the set of classifiers we are optimizing over, and $\ell(f(x),y)={\bf 1}_{\{f(x)\neq y\}}$ is the $0/1$ loss quantifying the discrepancy between the predicted label $f(x)$ and the truth.

	For the purpose of better generalization and robustness, we often desire an ERM solution with large margin. A classifier $f$ has margin $\epsilon$ with respect to some $p$-norm, if $(x,y) \in S$ then for any $\delta \in \real^d$ with $\|\delta\|_p \leq \epsilon$, $f(x) = f(x+\delta) = y$. While margin can be explicitly enforced through constraints or regularization for linear classifiers, doing so efficiently and provably for general classifiers remains a challenging open problem.
	Since data augmentation has had success in offering better robustness in practice, we ask the following question:
	
%	 To help capture this, we make no assumptions on the ERM solver whatsoever, so that it may find a model with as small a margin as possible. We ask the following question:

	\begin{center}
	\textit{Can data augmentation guarantee non-zero margin?}
	\end{center}

%	One could even ask if data augmentation leads to positive worst-case margin.
	That is, can we use an augmented data set $S^{\aug}$, such that by applying any ERM to it, the output classifier $\mA(S^{\aug})$ has some margin? Figure~\ref{fig:data_aug} provides a sketch of this problem for linear classification.
	\begin{figure}[h]
		%	{r}{0.28\textwidth}
			\vspace{-0.1cm}
		\centering
		\includegraphics[width=0.3\textwidth]{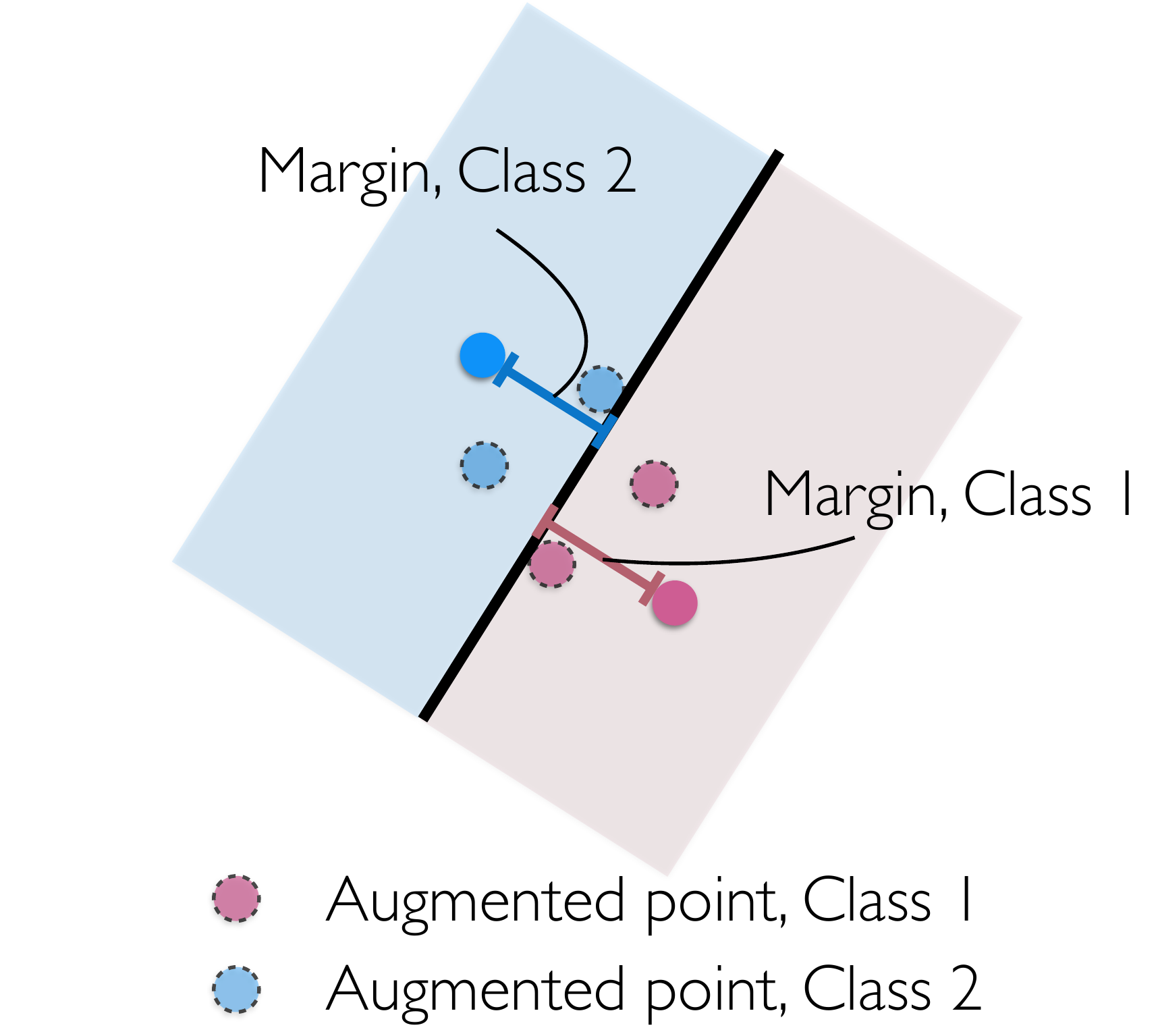}
			\vspace{-0.3cm}
		\caption{A linearly separable data set with two data points, each in its own class, and two input dimensions. If we wish to guarantee a positive margin for all feasible linear separators, i.e., all linear ERMs, we need to augment the training set with additional data points. Otherwise, a linear separator exists with zero margin.
		}
		\label{fig:data_aug}
			\vspace{-0.3cm}
	\end{figure}
	
	\paragraph{Lower bounds on the number of augmentations.} We first consider linear classification of linearly separable data. 
	We develop lower bounds on the number of augmented data points needed to guarantee that any linear separator of the augmented data has positive margin with respect to the original data set.
	We show that in $d$ dimensions, $d+1$ augmented data points are necessary for any data augmentation strategy to achieve positive margin. Moreover, there is some strategy that achieves the best possible margin with only $d+1$ augmented points. However, if the augmented points are formed by bounded perturbations of the training set, we need at least as many augmented data points as true training points to ensure positive margin.

	\paragraph{Upper bounds for additive random perturbations.} In practice, many data augmentation methods employ random perturbations, including random crops, rotations, and additive noise. 
	As a first step towards analyzing these methods, we focus on the setting that the augmented data set is formed by adding spherical random noise to the original training data.
	We specifically quantify how the dimension of the data, the number of augmentations per data point, and their norm can impact the worst-case margin.
	Our results show that if the norm of the additive noise is proportional to the margin, then the number of augmented data points must be exponential to ensure a constant factor approximation of the best possible margin.
	However, if the norm of the additive noise is carefully chosen, then polynomially many augmentations are sufficient to guarantee that any sperateor of the augmented data set has margin that is a constant fraction of the max margin of the original data set.

	\paragraph{Nonlinear classification and margin.} Finally, we extend our results to nonlinear classifiers that assign the same label within small convex hulls of the training data. We provide lower bounds on the number of augmentations needed for such ``respectful'' classifiers to achieve positive margin, and also analyze their margin under random DA methods. Despite respectful classifiers being significantly more general than linear ones, we show that their worst-case margin after augmentation can be comparable to that of linear classifiers.

\subsection{Related Work}

	DA is closely related to robust optimization methods \cite{xu2009robustness, caramanis201214, sinha2018certifying, wong2018provable}.
	While DA aims at improving model robustness via finitely many perturbations of the input data, robust optimization methods solve robust versions of ERM, which typically involve worst-case perturbations over infinite sets. Our work has particularly strong connections to \citet{xu2009robustness}, which shows that regularized SVMs are equivalent to robust versions of linear classification. Our results can be viewed as attempting to train robust models without the need to perform robust optimization. 

	Our work may also be viewed as quantifying the robustness of classifiers trained with DA against adversarial (i.e., worst-case) perturbations. Many recent works have analyzed the robustness of various classifiers to adversarial perturbations from a geometric perspective. \citet{fawzi2016robustness} introduce a notion of semi-random noise and study the robustness of classifiers to this noise in terms of the curvature of the decision boundary. \citet{moosavi2018robustness} also relate the robustness of a classifier to the local curvature of its decision boundary, and provide an empirical analysis of the curvature of decision boundaries of neural networks. \citet{fawzi2018analysis} relate the robustness of a classifier to its empirical risk and show that guaranteeing worst-case robustness is much more difficult than robustness to random noise. \citet{FranceschiFF18} provide a geometric characterization of the robustness of linear and ``locally approximately flat'' classifiers. Their results analyze the relation between the robustness of a classifier to noise and its robustness to adversarial perturbations. 
%	Our results actually give explicit bounds on the adversarial robustness of any linear separator of a dataset.
% !TeX root = icml_data_aug.tex

\begin{figure*}[h]
	\begin{subfigure}[b]{0.33\textwidth}
		\begin{center}
		\includegraphics[width=0.7\textwidth]{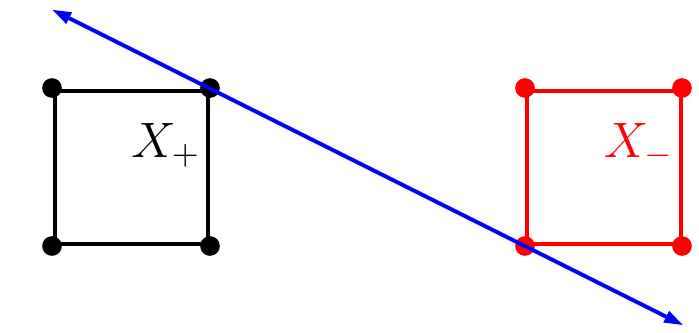}
		\vspace{0.5cm}
		\end{center}
		\caption{}
		\label{fig:a}
	\end{subfigure}%
	\begin{subfigure}[b]{0.33\textwidth}
		\begin{center}
		\includegraphics[width=\textwidth]{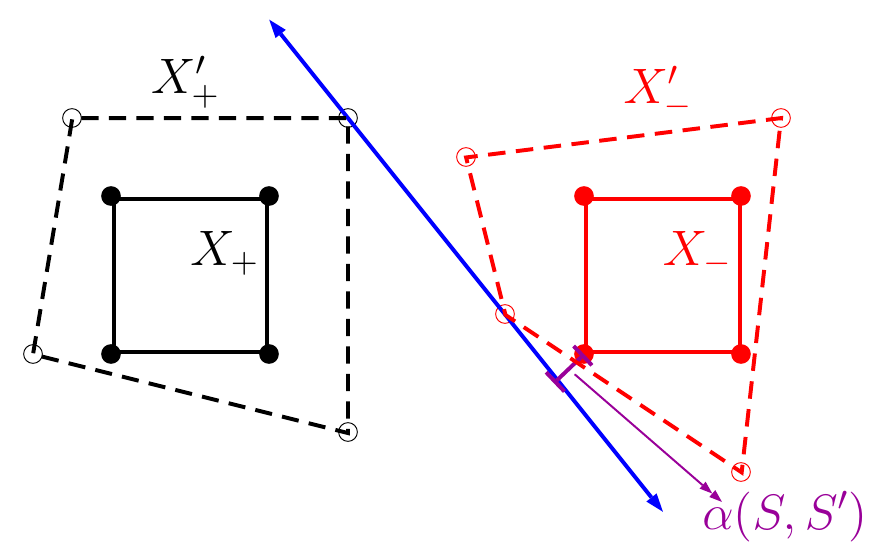}
		\end{center}
		\caption{}
		\label{fig:b}
	\end{subfigure}%
	\begin{subfigure}[b]{0.33\textwidth}
		\begin{center}
		\includegraphics[width=0.9\textwidth]{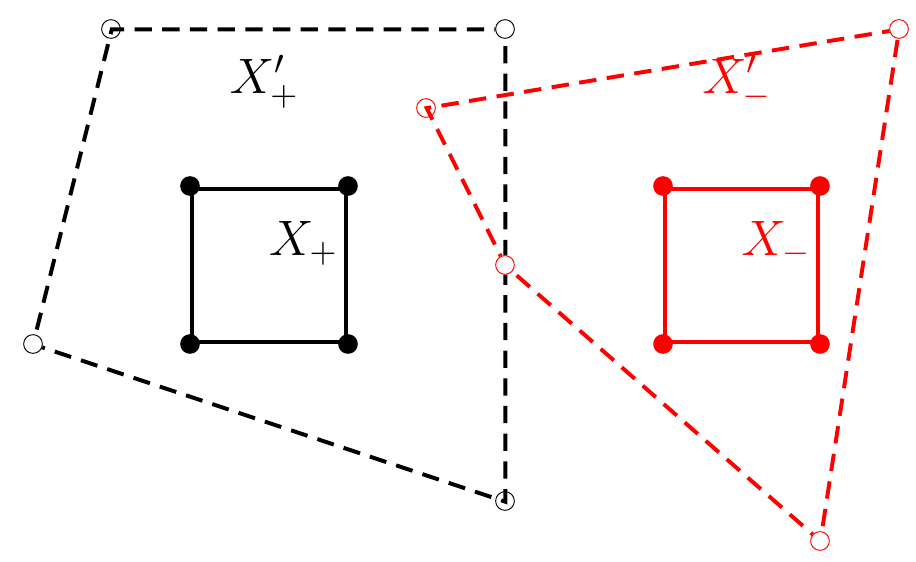}
		\vspace{-0.3cm}
	    \end{center}
		\caption{}
		\label{fig:c}
	\end{subfigure}%
	\caption{\small
				Solid dots represent the true data points and hollow dots represent artificial data points. Convex hulls of the true and augmented data are represented by solid and dashed lines, respectively. Classifiers are shown in blue. (a) Without DA, we may obtain a zero margin classifier. (b) Carefully chosen augmentations can guarantee positive margin. (c) Large augmentations may violate linear separability.}
  \label{fig:augmentations}
\end{figure*}

\section{Margin via Data Augmentation}\label{sec:prelim}

    Our work aims to quantify the potential of DA to guarantee margin for generic ERMs.
    We first examine linear classification on linearly separable data, and then extend our results to nonlinear classification.
    Although we can find max-margin linear classifiers efficiently through quadratic programming~\cite{shalev2014understanding}, generalizing this to nonlinear classifiers has proved difficult; if this was a simple task for neural networks, the problem of adversarial examples would be non-existent. Hence linear classification serves as a valuable entry point for our study of data agumentation.

    We first introduce some notation. Let $A, B \subseteq \real^d$, $x, y\in \real^d$, and $r \geq 0$. Let $d(x,y)$ denote the $\ell_2$ distance between $x,y$, and let $d(A,B) = \inf_{x \in A, y \in B}d(x,y)$. Define $A_r := \{z \in \real^d~|~d(z,A) \leq r\}$. Let $|A|$, $\int(A)$, and $\conv(A)$ denote the cardinality, interior, and convex hull of $A$. Let $\mS^{d-1}$ denote the unit sphere in $\real^d$, and for $r > 0$ let $r\mS^{d-1}$ denote the sphere of $r$.

    Let $S \subseteq \real^d \times \{\pm 1\}$ be our training set. For $(x,y) \in S$, $x$ is the feature vector, and $y \in \{\pm 1\}$ is the label. For any such $S$, we define
    \begin{equation}\label{x_eq}
    X_+ = \{x~|~(x,1) \in S\}, X_- = \{x~|~(x,-1) \in S\}.\end{equation}

    \paragraph{Linear classification.} We next recall some background on linear classification. 
     As in Section \ref{sec:problem_setup}, we assume we have access to an algorithm $\mA$ that solves the ERM problem over the set of linear classifiers.

    A linear classifier is a function of the form $h(x) = \sgn(\langle w,x\rangle + b)$, for $w \in \real^d, b \in \real$. We often identify $h$ with the hyperplane $H = \{x~|~\langle w,x\rangle + b = 0\}$. We say that $h$ linearly separates $S$ if $\forall x \in X_+, h(x) \geq 0$ and $\forall x \in X_-, h(x) \leq 0$. If such $h$ exists, $S$ is linearly separable. Let $\mH(S)$ denote the set of linear separators of $S$.
    
    \paragraph{Margin.} Suppose $S$ is linearly separable. The {\it margin} of a linear separator $h \in \mH(S)$ is defined as follows:

    \begin{definition}\label{def:linear_margin}
      The margin of a linear separator $h(x) = \sgn(\langle w,x\rangle + b)$ with associated hyperplane $H$ is
      $$\gamma_h(S) = \inf_{(x,y) \in S} d(x,H)
      = \inf_{(x,y) \in S} \dfrac{\left|\langle w,x\rangle + b\right|}{\|w\|_2}.$$
      We define $\gamma_h(S) = -\infty$ if $h$ does not linearly separate $S$.
    \end{definition}

    If $S$ is linearly separable, there is a linear classifier $h^*$ corresponding to $(w^*,b^*)$ with maximal margin $\gamma^*$. This classifier is the most robust linear classifier with respect to bounded $\ell_2$ perturbations of samples in $S$.
 
  In this work, we analyze
  the margin of ERMs that are trained without any explicit margin constraints or regularization. Let $S$ denote the {\it true dataset}. To achieve margin, we create an {\it artificial dataset} $S'$.
  We then assume  we have access to an algorithm that outputs (if possible) a linear separator $h$ of the {\it augmented dataset} $S^{\aug} := S \cup S'$. We define $X_{\pm}', X_{\pm}^{\aug}$ analogously to $X_{\pm}$ in \eqref{x_eq}.

  We will analyze the margin of $h$ with respect to the true training data $S$. If $S$ is linearly separable and we add no artificial points, then some $h \in \mH(S)$ must have 0 margin. If $S'$ is designed properly, one might hope that $S^{\aug}$ is still linearly separable and that any $h \in \mH(S^{\aug})$ has positive margin with respect to $S$. The following notion formalizes this idea, illustrated in Figure~\ref{fig:augmentations}.

    \begin{definition}The worst-case margin of a linear separator of $S^{\aug}$ with respect to the original data $S$ is defined as
    \begin{align*}
    \alpha(S,S') = \min_{h \in \mH(S^{\aug})} \gamma_h(S).
    \end{align*}
    We define this to be $-\infty$ if $\mH(S^{\aug}) = \emptyset$.\end{definition}
  We are generally interested in the following question:

  \begin{question}How do we design $S'$ so that $\alpha(S,S')$ is as large as possible?\end{question}

  In Section \ref{sec:lower_bounds}, we analyze how large $S'$ must be to ensure that $\alpha(S,S')$ is positive. We show that $|S'| > d$ is necessary to ensure positive worst-case margin. Moreover, if $S'$ is formed via bounded perturbations of $S$, we need $|S'| \geq |S|$ to guarantee positive margin. In Section \ref{sec:random}, we analyze the setting where $S'$ is formed by spherical random perturbations of $S$ of radius $r$, a technique that mirrors random noise perturbations used in practice. 
  We show that if $r$ is not well-calibrated, exponentially many perturbations are required to achieve a margin close to $\gamma^*$. However, if $r$ is set correctly, then it suffices to have $|S'|$ polynomial in $n$ and $d$ to ensure that any linear separator of $S^{\aug}$ will achieve margin close to $\gamma^*$ on $S$. In Section \ref{sec:nonlinear_classifiers}, we generalize this notion to  a class of nonlinear classifiers, which we refer to as ``respectful'' classifiers, and derive analogous results to those described above. We show that this class includes classifiers of general interest, such as nearest neighbors classifiers.

% !TeX root = icml_data_aug.tex

\section{Linear Classifiers}\label{sec:linear_classifiers}

\subsection{How Much Augmentation Is Necessary?}\label{sec:lower_bounds}

  % !TeX root = icml_data_aug.tex

  Suppose $S$ is linearly separable with max-margin $\gamma^*$. We wish to determine the required size of $S'$ to ensure that $\alpha(S, S') > 0$. We first show that to achieve a positive worst-case margin, the total number of perturbations must exceed the ambient dimension.

  \begin{theorem}\label{thm:lower_bound_1}If $|S'| < d+1$, then $\alpha(S, S') \leq 0$.\end{theorem}
  
  Therefore, we need to augment by at least $d+1$ points to ensure positive margin. We now wish to understand what margin is possible using data augmentation. We have the following lemma.

  \begin{lemma}
  \label{lem:alpha_bound}
  Let $\gamma^*$ be the maximum margin on $S$. For all $S' \subseteq \real^d$, $\alpha(S, S') \leq \gamma^*.$
  \end{lemma}

  In fact, if we know the max-margin hyperplane, then $d+1$ points are sufficient to achieve $\alpha(S,S') = \gamma^*$.

  \begin{theorem}\label{thm:sufficient_lower_bound}
  Let $S$ be linearly separable with max-margin $\gamma^*$. Then $\exists S'$ such that $|S'| = d+1$ and $\alpha(S,S') = \gamma^*$.\end{theorem}

  The augmentation method in the proof (see Section \ref{proof_sufficient_lower}) requires explicit knowledge of the maximum-margin hyperplane. In practice, most augmentation methods avoid such global computations, and instead apply bounded perturbations to the true data. Recall that for $A \subseteq \real^d$, $A_r = \{x | d(x,A) \leq r\}$. For $S \subseteq \real^d \times \{\pm 1\}$, we define
  \begin{equation}\label{s_r_eq}
  S_r = \bigg((X_+)_r\times \{1\}\bigg) \bigcup \bigg((X_-)_r\times \{-1\}\bigg).\end{equation}
  If $S'$ is formed from $S$ by perturbations of size at most $r$, then $S' \subseteq S_r$. The following result shows that if $S' \subseteq S_r$, then $|S'| \geq |S|$ is necessary to guarantee that $\alpha(S,S') > 0$.

  \begin{theorem}\label{thm:n_bound}Fix $(n, m) \in \NN^2$ and $r > 0$. Then $\exists S \subseteq \real^d$ with $|X_+| = n$ and $|X_-| = m$, such that if $S' \subseteq S_r$, and $|X_+'| < n$, then $\alpha(S, S') = 0$.\end{theorem}

Figure \ref{fig:thm5} provides an illustration. Given $r$, we choose $X_+$ to lie on a parabola $P$ such that the tangent lines at these points are at distance at least $r$ from other points. We choose $X_-$ to be far enough below the $x$-axis so that these tangent lines linearly separate $X_+$ from $X_-^{\aug}$. Suppose we do not augment some point $s \in X_+$. Then the tangent at that point linearly separates $X_+$ from $X_-^{\aug}$, while being at distance $0$ away from $s$. Thus, we need augmentations at every point in $X_+$ to guarantee positive margin.

  \begin{figure}[!h]
	\begin{center}
	    \includegraphics[width=0.3\textwidth]{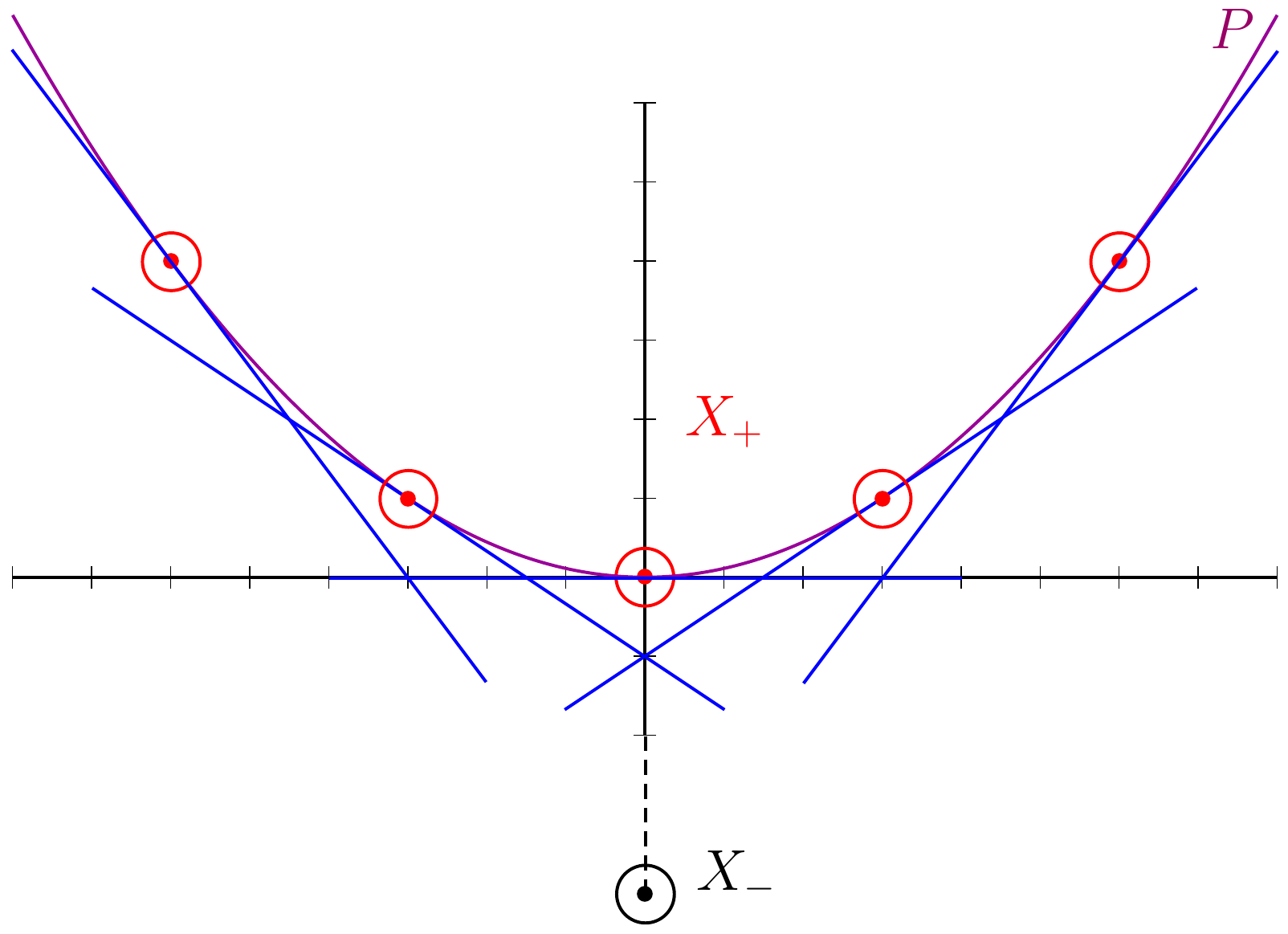}
	\end{center}
    \vspace{-0.5cm}
    \caption{Points in $X_+$ lie on the the parabola $P$ defined by $y = 9x^2$. The tangent at each point $s \in X_+$ does not intersect the ball of radius $r$ around any other point in $X_+$. We choose $X_-$ to have points far enough below the $x$-axis so that the tangents at $X_+$ separate $X_+'$ from any $X_-' \subseteq (X_-)_r$. Points in $X_+$ and their $r$-balls are in red, their tangents are in blue, and $X_-$ is in black.
    }
    \label{fig:thm5}
\end{figure}

\subsection{Random Perturbations}\label{sec:random}

	We now analyze the setting where $S'$ is formed by random perturbations of $S$. Our results reveal a fundamental trade-off between the size of perturbations, number of perturbations, margin achieved, and whether or not linear separability is maintained. If we construct many large perturbations, we may violate linear separability, but if we use too few perturbations that are too small in size, we may only achieve small margin guarantees.

	In the rest of this section, we assume that each point in $S'$ is of the form $(x+z,y)$ where $(x,y) \in S$ and $z$ is drawn uniformly at random from $r\mS^{d-1}$, the sphere of radius $r$. Due to the construction of $S'$, the following lemma about the inner products of random points on the sphere $\mS^{d-1}$ will be useful throughout.

	\begin{lemma}\label{lem:spher_cap}
		Let $a$ be a unit vector and $z$ be generated uniformly at random from the sphere of radius $\gamma$. Then with probability at least $1-e^{-d\epsilon^2/2\gamma^2}$, $\langle a,z\rangle \leq \epsilon.$
	\end{lemma}

	For further reference, see Chapter 3 of \cite{vershynin2011lectures}.

    \paragraph{Upper bounds on margin.}

    By Theorem \ref{thm:lower_bound_1}, we know that $|S'| \geq d+1$ is necessary to achieve positive margin on $S$. Since $S' \subseteq S_r$, we must have $\alpha(S,S') \leq r$. In general, we hope that high probability, $\alpha(S,S') \approx r$. We show below that the margin and perturbation size can be close only if $|S'|$ is exponential in $d$. The result follows using results on the measure of spherical cap densities to bound the distance between $S$ and the max-margin hyperplane.

    \begin{theorem}\label{thm:upper_bound_margin}For all $\delta \in (0,1)$, with probability at least $1-\delta$, we have
    $$\alpha(S,S') \leq \left(\sqrt{\dfrac{2\ln(|S'|)+2\ln(1/\delta)}{d}}\right)r.$$\end{theorem}

    This result shows that to achieve minimum-margin close to $r$, we need the number of perturbations to be exponential in $d$. Thus, if $r \approx \gamma^*$, we require exponentially many augmentations. However, by making $r$ much larger than $\gamma^*$, we may be able to achieve a large margin, provided linear separability is maintained.

  % !TeX root = icml_data_aug.tex
\paragraph{Maintaining linear separability.}

We now show that if $r$ is too large, the augmented sets will often not be linearly separable. Specifically, we show that when $S$ just has two points, if $r = \Omega(\sqrt{d}\gamma^*)$ and $|X_+'| = \Omega(d)$, then linear separability is violated with high probability. For Theorem \ref{thm:mainInSep}, suppose $S = \{(x_1,1), (x_2,-1)\}$ where $d(x_1,x_2) = 2\gamma^*$ (i.e., the max-margin is $\gamma^*$). 

  \begin{theorem}\label{thm:mainInSep}
    If $|X'_+| \geq 16d$ and $r \geq \frac{8e^2\sqrt{2d}}{\pi^{3/2}}\gamma^*$, with probability at least $1-2e^{-d/6}$, $S^{\aug}$ is not linearly separable.
  \end{theorem}

  To prove this, we first show that with high probability, there are $\Omega(d)$ points in $X_+'$ labeled $-1$ by the max-margin classifier. We then use estimates of when random points on the sphere are contained in a hemisphere to show that with high probability, the convex hull of the these points contains $x_2$. This analysis can be extended directly to the setting where $X_+$ and $X_-$ are contained in balls of sufficiently small radius compared to $\sqrt{d}\gamma^*$.

  On the other hand, we show that if $r$ is slightly smaller than $\sqrt{d}\gamma^*$, linear separability holds with high probability. 

  \begin{theorem}\label{thm:lin_sep}
  Suppose $S$ is linearly separable and $|S'|\leq N$. If $r \leq \beta^{-1/2}\sqrt{d/\log(N)}\gamma^*$ for $\beta > 1$, then with probability at least $1-N^{1-\beta}$, $S^{\aug}$ is linearly separable.
  \end{theorem}

  A short proof sketch is as follows: Let $w^*$ be a unit vector orthogonal to the max-margin hyperplane $H^*$. Suppose $(x+z,y) \in S'$ where $(x,y) \in S$ and $z$ is sampled uniformly on the sphere of radius $r$. By Lemma \ref{lem:spher_cap}, with high probability $\langle w^*,x+z\rangle$ will be close to $\langle w^*,x\rangle$, and so $x, x+z$ will fall on the same side of $H^*$. The result then follows by a union bound.

Theorems \ref{thm:mainInSep} and \ref{thm:lin_sep} together imply that if $r = \Omega(\sqrt{d}\gamma^*)$, we cannot hope to maintain linear separability. Instead, setting $r = O(\sqrt{d/\log N}\gamma^*)$, we will maintain linear separability with high probability. We will use the latter result in the next section to show that for such $r$, we can actually provide lower bounds on the adversarial margin $\alpha(S,S')$ achieved.

    \paragraph{Lower bounds on margin.}\label{lowerBoundMargin}

    By Theorem \ref{thm:upper_bound_margin}, we know that if $r \approx \gamma^*$, we need $N$ to be exponential in $d$ to achieve a margin close to $\gamma^*$. By Theorem \ref{thm:lin_sep}, we can set $r$ to be as large as $O(\sqrt{d/\log N}\gamma^*)$ and maintain linear separability. We might hope that in this latter setting, we can achieve a margin close to $\gamma^*$ with substantially fewer points than when $r \approx \gamma^*$.

    Suppose $S'$ is formed by taking $N$ perturbations of each point in $S = \{(x_i,y_i)\}_{i \in [n]}$. Formally, for $i \in [n], j \in [N]$ let $z_i^{(j)}$ be drawn uniformly at random from $r\mS^{d-1}$. Then,
    \begin{equation}\label{eq:alg_data_aug}
    S' = \{(x_i + z_i^{(j)}, y_i)\}_{i \in [n], j \in [N]}.\end{equation}

    We show following theorem:

    \begin{theorem}\label{thm:lower_bound_margin_general}
    Suppose $S$ is linearly separable with max-margin $\gamma^*$. Let $S'$ be as in \eqref{eq:alg_data_aug}. There is a universal constant $C$ such that if $N \geq Cd$ and $r \leq \beta^{-1/2}\sqrt{d/\log N}\gamma^*$ for $\beta > 1$, then with probability at least $1-ne^{-d}-nN^{1-\beta}$, we have
    $$\alpha(S,S') \geq \dfrac{1}{2\sqrt{2}}\sqrt{\dfrac{\log(N/d)}{d}}r.$$\end{theorem}

    Taking $r = \beta^{-1/2}\sqrt{d/\log N}\gamma^*$ and $\beta$ sufficiently large, we can ensure that the worst-case margin among linear separators is a constant fraction of the max-margin. Thus, with high probability, we can achieve a constant approximation of the best possible margin with $|S'| =  O(nd^2)$. While Theorems \ref{thm:lower_bound_1} and \ref{thm:n_bound} indicate that $|S'|$ should grow linearly in $n$ and $d$, determining whether $O(nd^2)$ is tight for some $S$ is an open problem.

    \begin{remark}\label{rem1}
    Theorem \ref{thm:lower_bound_margin_general} can be extended to the setting where we only take perturbations of each point in a $\tau$-cover of $X_+$ and $X_-$. Recall that $A$ is a $\tau$-cover of $B$ if $\forall x \in B, \exists x' \in A$ where $d(x,x') \leq \epsilon$. The same result (with the constant $2\sqrt{2}$ replaced by $4\sqrt{2}$) holds when $S'$ is formed according to \eqref{eq:alg_data_aug}, but with $S$ replaced by $A_+\times\{1\} \cup A_-\times\{-1\}$ where $A_+, A_-$ are $\tau$-covers of $X_+, X_-$ for
    \begin{equation}\label{eq:epsilon_bound}
    \tau = \frac{1}{4\sqrt{2}}\sqrt{\frac{\log(N/d)}{d}}r.\end{equation}
    Thus, we only need $|S'| = O(md^2)$ perturbations, where $m = \max\{|X_+|,|X_-|\}$. When $S$ is highly clustered, this could result in a much smaller sample complexity, as $m$ may be much smaller than $n$.
    \end{remark}

    To give a sketch of the proof, suppose $(0,1) \in S$. Thus, $S'$ contains $N$ points of the form $(z_i,1)$ where $z_i \sim r\mS^{d-1}$. We wish to guarantee that any linear separator, with associated hyperplane $H$, has some margin at 0. Consider $K = \conv(\{z_i\}_{i \in [N]})$. Since each $z_i$ has label 1, we know that $H$ cannot intersect the interior of $K$. Then, if $0$ is in the interior of $K$, then $H$ has positive margin at 0. In fact, we extract a strengthening of this from the proof of Lemma 3.1 of \cite{alonso2008isotropy}:

    \begin{lemma}\label{lma:sphere_inclusion}
      Let $z_1,\ldots, z_N$ be drawn uniformly at random on $r\mS^{d-1}$. Let $K=\text{conv}(z_1,\ldots, z_N)$. Then there exists a constant $C>0$ such that if $N \geq Cd$, then
      $$\PP\left( \frac{1}{2\sqrt 2}\sqrt{\frac{\log(N/d)}{d}} \mB_r(0)\not\subseteq K \right)\leq e^{-d}.$$
    \end{lemma}

    Thus, with high probability $\mB_\rho(0) \subseteq K$ where $\rho = \Omega(\sqrt{\log(N/d)/d}r)$. The margin of $H$ at 0 is therefore at least $\rho$. Applying Theorem \ref{thm:lin_sep}, we derive Theorem \ref{thm:lower_bound_margin_general}. A pictorial explanation of the proof is given in Figure \ref{fig:main_thm}.

    \begin{figure}[h]
        \vspace{-0cm}
        \centering
        \includegraphics[width=0.3\textwidth]{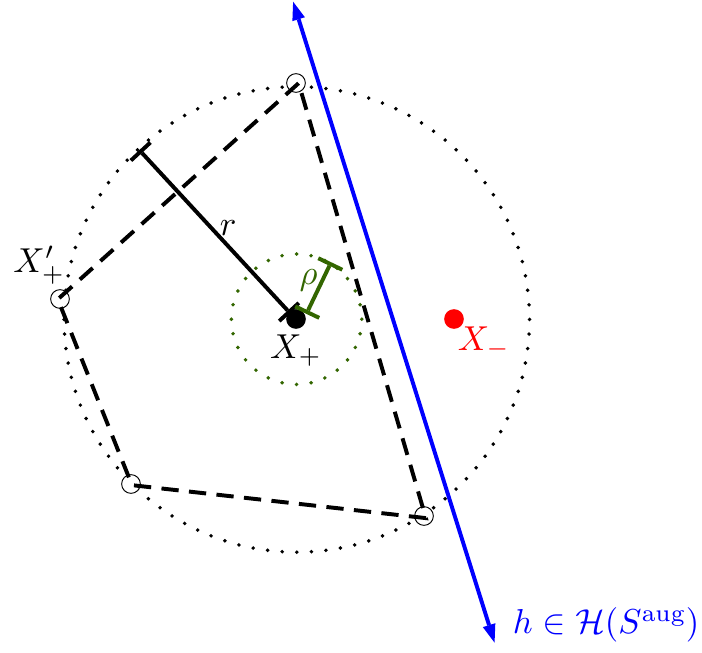}
        \vspace{-0.3cm}
        \caption{A pictorial explanation of the proof of Theorem \ref{thm:lower_bound_margin_general}. Suppose $X_+'$ is drawn uniformly at radius $r$ from $X_+$. With $r$ as in the theorem statement, with high probability $X_+'$ will not prevent linear separability of $S^{\aug}$. Moreover, with high probability $\conv(X_+')$ will contain a ball of radius $\rho$ around each point in $X_+$. This then implies that any $h \in \mH(S^{\aug})$ has margin at least $\rho$.}
        \label{fig:main_thm}
        \vspace{-0.3cm}
    \end{figure}

\section{Nonlinear Classifiers}\label{sec:nonlinear_classifiers}

	We now consider more general binary-valued classifiers. Given $S \subseteq \real^d \times \{\pm 1\}$, a classifier $f: \real^d \to \{\pm 1\}$ separates $S$ if $f(x) = y$ for all $(x,y) \in S$. Let $\mR(S)$ denote the collection of separators of $S$. If $\mR(S)$ is non-empty, we say that $S$ is separable. Given $f: \real^d \to \{\pm 1\}$, we define a generalization of the notion of margin in \ref{def:linear_margin}.

	\begin{definition}
		If $f \in \mR(S)$, its margin on $S$ is given by
		$$\gamma_f(S) := \min_{(x,y) \in S} d(x, f^{-1}(-y)).$$
		We define $\gamma_f(S) = -\infty$ if $f \notin \mR(S)$.
	\end{definition}
	Suppose we have a function class $\mF$ and we wish to find an ERM of the $0-1$ loss on $S$ (more generally, any nonnegative loss function where $\ell(f(x),y) = 0$ iff $f(x) = y$). The set of ERMs is simply $\mR(S) \cap \mF$.

	To find ERMs with positive margin, we will again form a perturbed dataset $S'$, and then find some ERM of $S^{\aug} = S \cup S'$. We define the margin of $f$ with respect to $S$ and $S'$ as follows.

	\begin{definition}
		The margin $\gamma_f(S,S')$ of $f$ with respect to $S, S'$ is defined by $\gamma_f(S,S') = \gamma_f(S)$ if $f\in \mR(S^{\aug})$ and $-\infty$ otherwise.
	\end{definition}

	If $S^{\aug}$ is separable and $\mF$ is sufficiently expressive, one can always find an ERM with zero margin. Instead, we will restrict to a collection of functions that is expressive, but still have meaningful margin guarantees. We refer to these as \textit{respectful} functions.
	
	\paragraph{Respectful classifiers.} If $x_1, x_2 \in \real^d$ are sufficiently close and have the same label, it is reasonable to expect a well-behaved classifier to assign the same label to every point between $x_1$ and $x_2$. In fact, \cite{fawzi2018empirical} shows that empirically, state-of-the-art deep nets often remain constant on straight lines connecting different points of the same class. For a linear classifier $f$ labels all points in $A$ as $1$, we know that $f$ assigns 1 to the entire set $\conv(A)$. With this in mind, we give the following definition:

	\begin{definition}
		A function $f: \real^d \to \{\pm 1\}$ is respectful of $S$ if $\forall x \in \conv(X_+), f(x) = 1$ and $\forall x \in \conv(X_-), f(x) = -1$.
	\end{definition}

	Intuitively, $f$ must respect the operation of taking convex hulls of points with the same label. However, assigning all of $\conv(X_{+})$ and $\conv(X_{-})$ the same label is a relatively strict condition. To relax this condition, we define a class of functions that are respectful only on small clusters of points. Recall the notion of a circumradius:

	\begin{definition}The circumradius $R(A)$ of a set $A \subseteq \real^d$ is the radius of the smallest ball containing $A$.\end{definition}

	We now define $\epsilon$-respectful classifiers:
	
	\begin{definition}
		For $\epsilon \in [0,\infty]$, we say that a classifier $f: \real^d \to \{\pm 1\}$ is $\epsilon$-respectful of $S$ if $\forall A \subseteq X_+$ such that $R(A) \leq \epsilon$, and $\forall x\in \conv(A)$, $f(x) = 1$; and $\forall B \subseteq X_-$ such that $R(B) \leq \epsilon$, and $\forall x \in \conv(B)$, $f(x) = -1$. Let $\mR_\epsilon(S)$ denote the set of $\epsilon$-respectful classifiers.
	\end{definition}

	An illustration is provided in Figure \ref{fig:locallyConvex}. Note that the set of separators of $S$ is simply $\mR_0(S)$, and the set of respectful classifiers is $\mR_\infty(S)$. Smaller values of $\epsilon$ lead to more expressive function classes $\mR_\epsilon(S)$. We now show that this definition includes some function classes of interest:
	\begin{figure}
		\vspace{-0.1cm}
		\centering
		\includegraphics[width=0.45\textwidth]{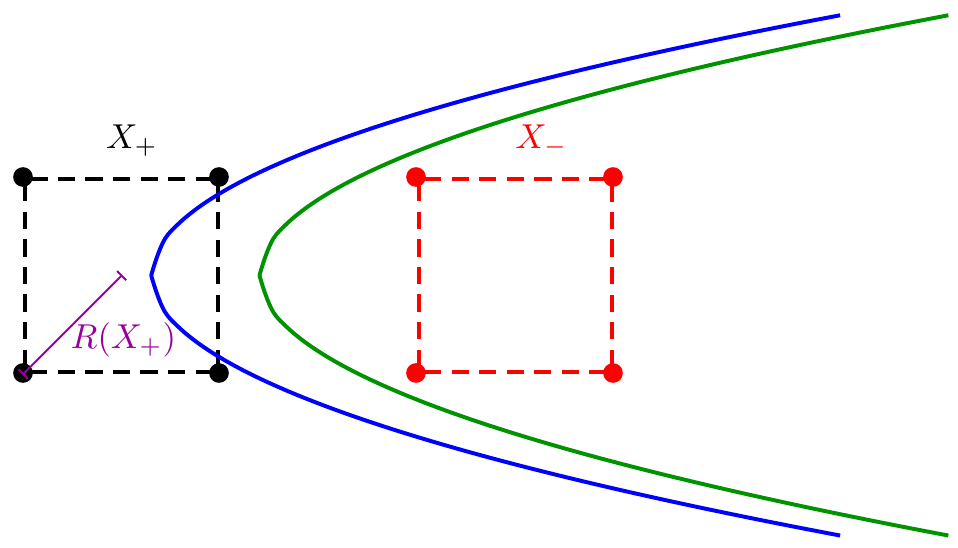}
		\vspace{-0.2cm}
		\caption{
			Suppose that $R(X_+) \leq \epsilon$. The classifier with a blue decision boundary is not $\epsilon$-respectful of $S$, but the classifier with a green decision boundary is $\epsilon$-respectful of $S$.
		}
		\label{fig:locallyConvex}
		\vspace{-0.3cm}
	\end{figure}

	\begin{example}[Linear Classifiers]
		Recall that $\mH(S)$ is the set of linear separators of $S$. It is straightforward to see that such functions are respectful of $S$, so $\mH(S) \subseteq \mR_\infty(S)$. By the hyperplane separation theorem (see Lemma \ref{lem:hyperplane_separation}), we have $\mH(S) \neq \emptyset$ if and only if $\mR_\infty(S) \neq \emptyset$. In general, $\mH(S)$ is a proper subset of $\mR_\infty(S)$.
	\end{example}

	\begin{example}[Nearest Neighbor]\label{ex:nn}
		Let $f_{NN}$ denote the $1$-nearest neighbor classifier on $S$: For $x \in \real^d$, we have $f_{NN}(x) = 1$ if $d(x,X_+) \leq d(x,X_-)$, and $f_{NN}(x) = -1$ otherwise. For $\epsilon \in [0,\frac{d(X_+,X_-)}{2})$, we can argue that $f_{NN} \in \mR_\epsilon(S)$, as follows: Suppose $x \in \conv(A)$ where $A \subseteq X_+$ and $R(A) \leq \epsilon$. Then $d(x,X_+) \leq \epsilon$. For all $u \in X_-$, we have $d(u,X_+) \geq d(X_+,X_-)$, so $d(x,u) > \frac{d(X_+,X_-)}{2}$. Hence, $f_{NN}(x) = 1$.
	\end{example}

	We now consider the following adversarial problem. Given $S$, we form a perturbed version $S'$. An adversary can pick an $\epsilon$-respectful classifier $f \in \mR(S^{\aug})$. The smaller the value of $\epsilon$, the more powerful the adversary. We hope that no matter which $f$ the adversary chooses, the value of $\gamma_f(S,S')$ is not too small.

	We first provide bounds on how large $S'$ must be to ensure a positive margin, and then derive results for random perturbations when $S$ is (non)-linearly separable. Our results are versions of Theorem \ref{thm:lower_bound_margin_general} for respectful classifiers. Finally, we will show that for respectful classifiers, our bounds for random perturbations are tight up to constants for some $S$.

\subsection{How Much Augmentation Is Necessary?}

		We first show that for any $\epsilon \in [0,\infty]$, we must have $|S'| > 2d$ in order to achieve a positive margin.

		\begin{theorem}\label{thm:nonlinear_lower_bound}
		Suppose $S$ is separable. If $|X_+'| \leq d$ or $|X_-'| \leq d$, then for any $\epsilon \in [0,\infty]$, either $\mR_\epsilon(S^{\aug}) = \emptyset$, or $\exists f \in \mR_\epsilon(S^{\aug})$ such that $\gamma_f(S,S') = 0$.
		\end{theorem}

		Suppose we limit ourselves to bounded perturbations of $S$, so that $S' \subseteq S_r$ for some $r > 0$. We will show that in this setting, we may need as many as $|S|(d+1)$ perturbations to guarantee a positive margin.

		\begin{theorem}\label{thm:nonlinear_lower_bound_2}
		For all $n \geq 1$ and $\epsilon, r \in (0,\infty)$, there is some $S$ of size $n$ such that if $|S'| \leq |S|(d+1)$, then $\exists f \in \mR_\epsilon(S^{\aug})$ such that $\gamma_f(S,S') = 0$.
		\end{theorem}

		Next, we consider the problem of ensuring a positive margin with bounded perturbations. The following lemma shows that if $\epsilon < r$, there is some $S$ such that the adversary can find a zero margin classifier for any $S' \subseteq S_r$.

		\begin{lemma}\label{lem:gamma_epsilon}
		For any $\epsilon \in (0,\infty)$ and $r > \epsilon$, there is $S$ such that for any $S' \subseteq S_r$, $\exists f\in \mR_\epsilon(S^{\aug})$ such that $\gamma_f(S,S') = 0$.
		\end{lemma}

		Therefore, for $S' \subseteq S_r$, to ensure that any $f \in \mR_\epsilon(S^{\aug})$ has positive margin, we need $r \leq \epsilon$, $|S'| \geq 2d+2$, and $|S'| \geq |S|(d+1)$. In fact, these three conditions are sufficient to ensure positive margin.
		
		\begin{theorem}\label{thm:suff_nonlinear}
		For any $S$, if $\epsilon \in (0,\infty]$ and $r \leq \epsilon$, then $\exists S' \subseteq S_r$ with $|S'| = |S|(d+1)$, such that $\forall f \in \mR_\epsilon(S^{\aug})$, $\gamma_f(S,S') > 0$.
		\end{theorem}

		While this theorem does not guarantee that $\mR_\epsilon(S^{\aug}) \neq \emptyset$, we will show in Lemma \ref{lem:nonempty_overall} that if $S' \subseteq S_r$ for $r \leq \epsilon < \frac{d(X_+,X_-)}{4}$, then $\mR_\epsilon(S^{\aug})$ is guaranteed to be nonempty.

\subsection{Random Perturbations}

	We now analyze how random perturbations affect the margin of $\epsilon$-respectful classifiers. Just as in the linear setting, we focus on the case where the points in $S'$ are of the form $(x+z,y)$ where $z$ is drawn uniformly at random from the sphere of radius $r$. We provide lower bounds on the margin that are analogous to the linear setting, and show that our margin bounds are tight up to constants in some settings.

	\paragraph{Linearly separable data.}We first show that when $S$ is linearly separable and we perform random augmentations, the results in Section~\ref{lowerBoundMargin} still hold, even though the adversary is allowed to select classifiers in the larger set $\mR_\infty$.

	    \begin{theorem}\label{thm:nonlinear1}
	    Let $S'$ be generated as in \eqref{eq:alg_data_aug}. There is a universal constant $C$ such that if $N \geq Cd$ and $r \leq \beta^{-1/2}\sqrt{d/\log N}\gamma^*$ for $\beta > 1$, then with probability at least $1-ne^{-d}-nN^{1-\beta}$, we have $\mR_\infty(S) \neq \emptyset$. Furthermore, $\forall f \in \mR_\infty(S^{\aug})$, we have
	    $$\gamma_f(S,S') \geq \dfrac{1}{2\sqrt{2}}\sqrt{\dfrac{\log(N/d)}{d}}r.$$\end{theorem}

	    The proof uses a generalization of Theorem \ref{thm:lower_bound_margin_general} to respectful functions. We show in Theorem \ref{thm:nonlinear2} that this bound is tight up to constants under certain assumptions on $S$.

	    As in the linear case, a perturbation radius of $r = O(\sqrt{d}\gamma^*)$ is necessary to maintain separability. Suppose $S = \{(x_1,1),(x_2,-1)\}$ with $d(x_1,x_2) = 2\gamma^*$ and $S'$ is as in \eqref{eq:alg_data_aug}. Applying the hyperplane separation theorem and Theorem \ref{thm:mainInSep}, we have the following result:

		\begin{theorem}\label{thm:insep_lc}
		If $N \geq 16d$ and $r \geq \frac{8e^2\sqrt{2d}}{\pi^{3/2}}\gamma^*$, then
		$$\PP(\mR_\infty(S^{\aug}) = \emptyset) \geq 1 - 2e^{-d/6}.$$
		\end{theorem}

		In short, spherical random data augmentation behaves similarly when the adversary selects linear classifiers or classifiers in $\mR_\infty(S^{\aug})$, both in terms of margin achieved and upper bounds on perturbation size to maintain separability.

	\paragraph{Nonlinearly separable data.}

		When $S$ consists of more than two points, the margin obtained by some $f \in \mR_\epsilon(S)$ may be much larger than the max-margin linear classifier. Moreover, $\mR_\epsilon(S)$ may be non-empty even though $S$ is not linearly separable. Thus, we would like to derive versions of the results in Section \ref{lowerBoundMargin} for settings where $S$ may not be linearly separable, but $\mR_\epsilon(S) \neq \emptyset$. In fact, if $\mR_\epsilon(S) \neq \emptyset$ and we generate $S'$ as in \eqref{eq:alg_data_aug}, we can derive the following theorem, comparable to Theorem \ref{thm:lower_bound_margin_general} above:

		\begin{theorem}\label{thm:nonlinear2}
			If $r \leq \epsilon$, then there is a universal constant $C$ such that if $N \geq Cd$, then with probability at least $1-ne^{-d}$, $\forall f \in \mR_\epsilon(S^{\aug})$,
			$$\gamma_f(S,S') \geq \frac{1}{2\sqrt{2}}\sqrt{\frac{\log(N/d)}{d}}r.$$
			Furthermore, if $\epsilon < \frac{d(X_+,X_-)}{4}$ then $\mR_\epsilon(S^{\aug}) \neq \emptyset$.
		\end{theorem}

		The first part of the proof proceeds similarly to that of Theorem \ref{thm:lower_bound_margin_general}, using the definition of $\epsilon$-respectful classifiers. For the second, we use nearest neighbor classifiers (as in Example \ref{ex:nn}) to construct $\epsilon$-respectful classifiers of $S^{\aug}$.

		Although $r \leq \epsilon < \frac{d(X_+,X_-)}{4}$ is sufficient to guarantee that $\mR_\epsilon(S^{\aug}) \neq \emptyset$, this may be overly conservative. Whereas Theorems \ref{thm:nonlinear1} and \ref{thm:insep_lc} provide a characterization of the range on $r$ for which $\mR_\infty(S^{\aug})$ is non-empty with high probability, a tighter characterization for $\epsilon < \infty$ remains open.

   \paragraph{Upper bounds on margin.} Finally, we show that for certain $S$, the margin bounds in Theorems \ref{thm:nonlinear1} and \ref{thm:nonlinear2} are tight up to constants. While it is as yet unknown whether Theorem \ref{thm:lower_bound_margin_general} is asymptotically tight, the increased expressive capability of respectful classifiers allows us to exhibit upper bounds on the worst-case margin matching the lower bounds above. Suppose $S = \{(x_1,1),(x_2,-1)\}$, and $S'$ is generated as in \eqref{eq:alg_data_aug}. We have the following result:
   
    \begin{theorem}\label{thm:nonlinear3}
        Fix $\epsilon \in [0,\infty]$ and $r > 0$. There are absolute constants $C_1, C_2$ such that if $N> d$ and $\mR_\epsilon(S^{\aug}) \neq \emptyset$, then with probability at least $1 - 2e^{-C_2 d \log(N/d)}$, $\exists f \in \mR_\epsilon(S^{\aug})$ such that
        \begin{equation}\label{eq:upper_circumradius}
        \gamma_f(S,S') \leq \sqrt{C_1\dfrac{\log(2N/d)}{d}}r.
        \end{equation}
    \end{theorem}

    The proof relies on estimates of the inradius of random convex polytopes from \cite{alonso2008isotropy}. The theorem can also be extended to settings where $X_+$ and $X_-$ are not singletons. Suppose we can decompose $X_+$ and $X_-$ into clusters $\{A_i\}_{i=1}^k$ and $\{B_j\}_{j=1}^l$ such that each cluster has size at most $m$, circumradius at most $O(\sqrt{\log(N/d)/d}r)$, and the distance between any two clusters is $\Omega(\epsilon)$. If $S'$ is generated as in \eqref{eq:alg_data_aug}, then with high probability there is some $f \in \mR_\epsilon(S^{\aug})$ satisfying \eqref{eq:upper_circumradius} where $N$ is replaced by $mN$.

% !TeX root = data_aug_main.tex

\section{Conclusion and Open Problems}
Data augmentation is commonly used in practice, since it significantly improves test error and model robustness.
In this work, we have analyzed the performance of data augmentation through the lens of margin.
We have demonstrated how data augmentation  can guarantee positive margin for unconstrained empirical risk minimizers. For both linear and nonlinear ``respectful'' classifiers, we provided lower bounds on the number of points needed to ensure positive margin, and analyzed the margin attained by additive spherical data augmentation.

There are several interesting open problems that we plan to tackle in the future. First, it would be interesting to theoretically analyze practical state-of-the-art augmentation methods, such as random crops, flips, and rotations. Such perturbations often fall outside our framework, as they are not bounded in the $\ell_2$ norm.
Another fruitful direction would be to examine the performance of adaptive data augmentation techniques. 
For example, robust adversarial training, (such as in \cite{madry2017towards}), can be viewed as a form of adaptive data augmentation. By taking a data augmentation viewpoint, we hope to derive theoretical benefits of using adversarial training methods.
One final direction would be to develop improved augmentation methods. In particular, we would like methods that can exploit domain knowledge and the geometry of the underlying problem in order to find models with better robustness and generalization properties.

\bibliographystyle{icml2019}
\balance
\bibliography{data_aug_camera}

\newpage
\nobalance
\appendix

\section{Mathematical Background}\label{sec:math_background}

	We first give some definitions related to convex geometry that we will use in the following proofs. For the following, we will consider sets in $\real^d$ under the $\ell_2$ topology. Let $S \subseteq \real^d$, and let $S^c$ denote its complement.

	\begin{definition}
		The convex hull of a set $S$ is the intersection of all convex sets containing $S$.
	\end{definition}

	\begin{definition}
		A point $x$ is in the interior of $S$ if there is an open ball centered at $x$ completely contained in $S$. The collection of interior points of $S$ is denoted $\intr(S)$.
	\end{definition}

	\begin{definition}
		A point $x$ is on the boundary of $S$ if every ball centered at $x$ has non-empty intersection with $S$ and $S^c$. The collection of boundary points is denoted $\partial S$.
	\end{definition}

	By definition, $d(x,S^c) > 0$ iff $x \in \intr(S)$, and $d(x,S^c) = 0$ otherwise. Given $a,b \in \real^d$, we will let $d(a,b)$ denote their $\ell_2$ distance. Similarly, for $A, B \subseteq \real^d$, we will let $d(A,B) = \inf_{a \in A, b \in B}d(a,b)$. For $a \in \real^d$, we will use $d(a, B)$ to denote $d(\{a\},B)$. Given $r \geq 0$, we let $A_r = \{x \in \real^d | d(x,A) \leq r\}$, and $\mB_r(x) = \{ z \in \real^d | d(x,z) \leq r\}$.

	We will also make use of the hyperplane separation theorem, originally due to Minkowski. For a more detailed reference, see \cite{boyd2004convex}.

	\begin{lemma}\label{lem:hyperplane_separation}
	Let $A, B$ be two disjoint convex subsets of $\real^d$. Then there exists some non-zero $v \in \real^d$ and $c \in \real$ such that $\langle x,v \rangle \geq c$ and $\langle y, v\rangle \leq c$ for all $x$ in $A$ and $y$ in $B$.
	\end{lemma}

\section{Proof of Results in Section \ref{sec:lower_bounds}}

    \subsection{Proof of Lemma \ref{lem:alpha_bound}}

  \begin{proof}
    If $S^{\aug}$ is not linearly separable, then $\alpha(S,S') = -\infty$ and the result follows. Thus, suppose $S^{\aug}$ is linearly separable. Then let $h \in \mH(S^{\aug})$ correspond to the hyperplane $H$. Since $S \subseteq S^{\aug}$, $h \in \mH(S)$. By definition,
    $$\alpha(S, S') \leq \min_{(x,y) \in S} d(x,H) \leq \min_{(x,y)\in S } d(x,H^*) = \gamma^*.$$
    where $H^*$ is the hyperplanes defined by the max-margin classifier of $S$.
  \end{proof}

\subsection{Proof of Theorem \ref{thm:lower_bound_1}}

  To prove this we will first prove the following lemma.
  \begin{lemma}\label{lem:lower_aux}
  Suppose $u_1,\ldots, u_n, v_1,\ldots, v_m$ are vectors in $\real^d$ with $n, m \geq 1$ and $m < d$. Suppose there is a vector $z \neq 0$ such that $\langle u_i, z\rangle \geq 0$, $\langle v_j,z\rangle \geq 0$ for all $i$ and $j$. Then there is a vector $z' \neq 0$ such that $\langle u_i,z'\rangle \geq 0$, $\langle v_j, z'\rangle \geq 0$ for all $i$ and $j$ and there is some $i$ such that $\langle u_i,z'\rangle = 0$.
  \end{lemma}

  \begin{proof}
    We induct on $n$. Suppose $n = 1$. If the vectors $u_1, v_1,\ldots, v_m$ do not span $\real^d$, then there is a non-zero vector $z$ such that $\langle u_1, z\rangle = 0$ and $\langle v_j, z\rangle = 0$ for all $j$, completing the proof. Otherwise, we may assume that $m = d-1$ and $u_1$ is not in the span of $v_1,\dots, v_{d-1}$. Thus, the matrix $Q$ whose rows are $u_1^T, v_1^T, \ldots, v_{d-1}^T$ is a $d\times d$ matrix of rank $d$. Therefore, there is some non-zero $z$ such that $Qz = [0,1,\dots,1]^T$. The vector $z$ satisfies the desired conditions.

    Suppose the result holds for $n = k-1$, and we have vectors $u_1,\ldots, u_k$, $v_1,\ldots, v_m$ such that $m < d$. Let $A$ denote the set of $x \in \real^d$ such that $\langle x, u_i\rangle, \langle x, v_j\rangle \geq 0$ for $1 \leq i \leq k,~1 \leq j \leq m$. Let $B$ denote the set of $x \in \real^d$ such that $\langle x, u_i\rangle, \langle x, v_j\rangle \geq 0$ for $1 \leq i \leq k-1,~1 \leq j \leq m$. By assumption, we know that there is some $z \in A$ such that $z \neq 0$. By the inductive hypothesis, we know there is a non-zero vector $w \in B$ and $l$ such that $1 \leq l \leq k-1$ and $\langle w, u_i\rangle = 0$.

    Let $H = \{x \in \real^d~|~\langle x, u_k \rangle = 0\}, H^+ = \{x \in \real^d~|~\langle x,u_k\rangle \geq 0\}$. If $w \in H^+$, then we are done. Otherwise, $\langle w, u_k\rangle < 0$. Since $w \notin H^+$ and $z \in H^+$, there is some $\lambda \in [0,1]$ such that the point $z_\lambda  = (1-\lambda) w + \lambda z$ satisfies $z_\lambda \in H$. Since $B$ is a closed convex set, we know that $z_\lambda \in B$. Therefore, $z_\lambda \in B\cap H \subseteq A$. Moreover, $\langle z_\lambda, u_k\rangle = 0$. It therefore suffices to show that $z_\lambda \neq 0$. Since $z, w \neq 0$, this can occur if and only if $w = -cz$ for some $c > 0$. But since $z,w \in B$, this would imply that $\langle z, u_i \rangle = 0$ for $1 \leq i \leq k-1$. In particular, $z$ then would satisfy the assumptions of the theorem, completing the proof.
  \end{proof}

  We can now prove Theorem \ref{thm:lower_bound_1}.

  \begin{proof}[Proof of Theorem \ref{thm:lower_bound_1}]
    If $S^{\aug}$ is not linearly separable, then $\alpha(S,S') = -\infty$. Otherwise, there is some $w \in \real^d, b \in \real$ such that for all $(x,y) \in S^{\aug}, y(\langle w,x\rangle + b) \geq 0$.

    Suppose $S = \{(x_i,y_i)\}_{i=1}^n$ and let $u_i = (y_ix_i,y_i) \in \real^{d+1}$. Analogously, suppose $S' = \{(x_j',y_j')\}_{j=1}^m$ and let $v_j = (y_j'x_j',y_j') \in \real^{d+1}$.

    % For each $(x,y) \in S$, define a vector $u_{x,y} \in \real^{d+1}$ by
    % $$u_{(x,y)}^T = [x^T~y], x \in S,~~u_x^T = [-x^T~-1], x \in T.$$
    % Similarly, for $x \in S' \cup T'$, define vectors $v_x \in \real^{d+1}$ by
    % $$v_x^T = [x^T~1], x \in S',~~v_x^T = [-x^T~-1], x \in T'.$$

    By construction, $y\langle w, x\rangle + b \geq 0$ for all $(x,y) \in S^{\aug}$ iff $z = (w,b) \in \real^{d+1}$ satisfies the following two conditions:
    \begin{equation}\label{eq:lower_cond1}
    \forall i \in [n],~\langle u_i, z\rangle \geq 0\end{equation}
    \begin{equation}\label{eq:lower_cond2}
    \forall j \in [m],~\langle v_j, z\rangle \geq 0\end{equation}

    Since $S^{\aug}$ is linearly separable, there is some non-zero vector $z \in \real^{d+1}$ satisfying \eqref{eq:lower_cond1} and \eqref{eq:lower_cond2}. Since $|S'| < d+1$, we can apply Lemma \ref{lem:lower_aux} to the vectors $\{u_i\}_{i \in [n]}$ and $\{v_j\}_{j \in [n]}$. Therefore, there is a non-zero vector $z \in \real^{d+1}$ satisfying \eqref{eq:lower_cond1} and \eqref{eq:lower_cond2}, and such that there is some $[i] \in S$ such that $\langle z,u_{i}\rangle = 0$.

    Let $z_{1:d} \in \real^d$ be the vector of its first $d$ coordinates, and let $c$ denote its last coordinate. Therefore, $y_i\langle z_{1:d},x_i\rangle + c = 0$, which implies that $(z_{1:d},c)$ has zero margin at $(x_i,y_i)$. It now suffices to show that $(z_{1:d},c)$ corresponds to a non-zero linear separator of $S^{\aug}$. By construction of \eqref{eq:lower_cond1}, \eqref{eq:lower_cond2}, we know that for all $(x,y) \in S^{\aug}$, $y\langle z_{1:d},x \rangle + c \geq 0$.

    It therefore suffices to show that $z_{1:d} \neq 0$ to show that $(z_{1:d},c)$ is actually a well-defined linear separator of $S^{\aug}$. If $z_{1:d} = 0$, then for any $(x,y) \in S$, $y\langle z, u_x\rangle + c = c \geq 0$. Since $X_+, X_-$ are both non-empty, this implies that $c = 0$, so $z = 0$, giving us a contradiction. Hence, $(z_{1:d},c) \in \mH(S^{\aug})$ and has zero margin on $S$, so $\alpha(S,S') \leq 0$.
  \end{proof}

\subsection{Proof of Theorem \ref{thm:sufficient_lower_bound}}\label{proof_sufficient_lower}

  \begin{proof}
    Let $H^*$ be the maximum margin separating hyperplane of $S$. Let $z_1,\ldots, z_d$ be in general position on $H^*$ and let $z_{d+1}$ lie in the interior of their convex hull. Let $S' = \{(z_1,1),\ldots, (z_d,1), (z_{d+1},-1)\}$. By construction, the max-margin classifier $(w^*,b^*)$ of $S$ satisfies $y(\langle w^*,x\rangle + b) \geq 0$ for all $(x,y) \in S^{\aug}$ and satisfies $\gamma_{(w^*,b^*)}(S) = \gamma^*$.

    Now, suppose that $(w,b)$ is a linear separator of $S^{\aug}$ whose associated hyperplane $H$ is not equal to $H^*$. Since $\forall i \in [d], \langle w, z_{i}\rangle + b \geq 0$, and $z_1,\ldots, z_d$ uniquely determine $H^*$, there must be some $j \in [d]$ such that $\langle w, z_j \rangle + b > 0$. Since the $z_j$ are in general position and $z_{d+1}$ is in their interior, this then implies that $\langle w,z_{d+1}\rangle + b > 0$. This contradicts $(w,b)$ being a linear separator of $S^{\aug}$. Therefore, $(w,b)$ is a linear separator of $S^{\aug}$ iff its associated hyperplane is $H^*$, in which case it has margin $\gamma^*$.
  \end{proof}

\subsection{Proof of Theorem \ref{thm:n_bound}}

    \begin{proof}
       First assume $r = 1$. We will construct $S$ such that we need augmentation at every point of $S$ to ensure $\alpha(S,S') > 0$. Below, we give the construction and analysis in $\real^2$.

        We will construct $X_+$ by taking points on the parabola $x_2 = x_1^2$ that are sufficiently far apart. The first point of $X_+$ is chosen to be the point $(s_0,s_0^2)$ such that $s_0 = 3$. Now, the $i$-th point of $X_+$ is chosen to be $(s_i,s_i^2)$ such that $s_i=2s_{i-1}+4$.

        Next, we calculate the tangent line $t_i$ to the curve $x_2=x_1^2$ at the $i^{th}$ point. This is given by the equation $x_2=2s_i x_1- s_i^2$. Now, the distance between $t_i$ and any point $(s_j,s_j^2)$ for $j \neq i$ from $X_+$ is given by

        $$d(t_i,s_j) = \left(\dfrac{s_j^2 + s_i^2}{2s_i} - s_j\right)\sin{(\theta)}$$

        where $\theta = \tan^{-1}(2s_i)$ is the slope of the tangent. Because all $s_i > 2$ by construction, $\sin{\theta}>1/2$.
        
        Thus, the distance from any point $s_j$ with $j \neq i$ to $t_i$ satisfies
        $$d(t_i,s_j) >\dfrac{1}{2}\left(\dfrac{s_j^2 + s_i^2}{2s_i} - s_j\right) \geq 1.$$
        The last inequality is true because we chose $s_i=2s_{i-1}+4$. Thus, the tangent at any point lies at a distance more than $1$ from other points. 

        We select $X_-$ to be any set of points far enough down the $y$ axis such that these tangents linearly separate $X_+$ and $(X_-)_r$. Therefore, these tangents linearly separate $X_+$ and $X_- \cup X_-'$ for any $X_-' \subseteq (X_-)_r$. In particular, $X_-'$ can be of any size, even infinite. Now, suppose that $|X_+'| < n$. Therefore, there is some $i$ such that $X_+'$ is of the form
        $$X_+' = \left\{ x + z~|~x\in X_+,~x\neq(s_i,s_i^2)\right\}.$$
        In other words, we augment $X_+$ without adding any perturbations around the $i^{th}$ point $(s_i,s_i^2)$. It is easy to see that the tangent at $(s_i,s_i^2)$ will still linearly separate $X_+^{\aug}$ and $X_-^{\aug}$. Moreover, this tangent has $\alpha(S,S')=0$ because the point $(s_i,s_i^2)\in X_+$ lies on it.

        Thus, this proves that in $\real^2$, there exist sets $X_+$ and $X_-$ with $|X_+| = n, |X_-| = m$, such that we need $|X_+'|\geq n$ to guarantee a positive $\alpha(S,S')$. For $r > 1$, we can modify the above construction by spacing out the points $s_i$ more. If we want to do this in $d > 2$ dimensions, we can either take points of the form $s_ie_j + s_ie_d$ where $e_i$ is the $i$-th standard basis vector. An analogous argument holds.
    \end{proof}

\section{Proof of Results in Section \ref{sec:random}}

        \subsection{Proof of Theorem \ref{thm:upper_bound_margin}}
	
    	\begin{proof}Fix $S$. If it is not not linearly separable, then the result is immediate. Otherwise, there is some maximum-margin classifier $(w^*,b^*)$ with maximum margin $\gamma^*$. Without loss of generality, we can rescale $(w^*,b^*)$ so that $\|w^*\|_2 = 1$, in which case we have that for all $(x,y) \in S$, $y(\langle w^*,x \rangle +b^*) \geq \gamma^*$.

        If $\epsilon \geq \gamma^*$, then $\alpha(S,S') \leq \epsilon$ by Lemma \ref{lem:alpha_bound}, in which case the theorem holds immediately. Otherwise, suppose $0 \leq \epsilon < \gamma^*$.

        Let $|S'| = N$. By assumption, the $i$-th point in $S'$ is of the form $(x_i+z_i,y_i)$ where $x_i \in X_+ \cup X_-$, $z_i$ is drawn uniformly at random from the sphere of radius $r$, and $y_i \in \{1,-1\}$. Suppose that there is an $\epsilon > 0$ such that for all $i \in [N]$, the following holds:

        \begin{equation}\label{eq:z_cond}
        y_i\langle w^*,z_i\rangle \geq -\epsilon\end{equation}

        We will show that if \eqref{eq:z_cond} holds for all $i \in [N]$, then $\alpha(S,S') \leq \epsilon$. Define $b' = b^*-\gamma^*+\epsilon$. For any $(x,y) \in S$, we then have
        \begin{equation}\label{upper_1_eq}
        \begin{split}
        y(\langle w^*,x\rangle+b')&= y(\langle w^*,x\rangle + b^*) - y\gamma^*+y\epsilon\\
        &\geq (1-y)\gamma^* + y\epsilon\\
        &\geq \epsilon.
        \end{split}
        \end{equation}      
        The second equation holds since $(x, y) \in S$ and by definition of $(w^*,b^*)$, and the third holds because $\gamma^* \geq \epsilon$ and $y = \pm 1$.

        Assume that \eqref{eq:z_cond} holds for all $i \in [N]$. For any such $i$, we then have
        \begin{equation}\label{upper_2_eq}
        \begin{split}
        &y_i(\langle w^*,x_i+z_i\rangle+b')\\
        &= y_i(\langle w^*,x_i\rangle + b') + y_i(\langle w^*,z_i\rangle)\\
        &\geq 0
        \end{split}
        \end{equation}
        where the last inequality holds by \eqref{upper_1_eq} and \eqref{eq:z_cond}. Taken together, \eqref{upper_1_eq} and \eqref{upper_2_eq} imply that if \eqref{eq:z_cond} holds for all $i \in [N]$, then $(w^*,b')$ linearly separates $S^{\aug}$, with margin at least $\epsilon$ on $S$.

        Thus, to show that $\alpha(S,S') \leq \epsilon$, it suffices to show that \eqref{eq:z_cond} holds. By Lemma \ref{lem:spher_cap} and the union bound, with probability at least $1-Ne^{-d\epsilon^2/2r^2}$, we have that for all $i$, \eqref{eq:z_cond} holds. Setting $1-Ne^{-d\epsilon^2/2r^2}$ equal to $1-\delta$ and solving for $\epsilon$, we derive the desired result.\end{proof}

      \subsection{Proof of Theorem \ref{thm:mainInSep}}\label{proof:mainInSep}

  We first prove an auxiliary lemma.

      \begin{lemma}\label{thmInsep}
          Fix some unit vector $a$. Suppose we sample $z_1, \ldots, z_N$ independently and uniformly at random on the sphere of radius $r = \frac{8e^2\gamma^* \sqrt{2d}}{\pi^{3/2}}$. If $N \geq 16d$, then with probability at least $1 - e^{-9(d-1)/8}$, at least $3d$ of the $z_i$ satisfy $\langle a,z_i\rangle \geq 2\gamma^*$.
      \end{lemma}

      \begin{proof}[Proof of Lemma \ref{thmInsep}]
            Given a point $x \in \real^d$, we will let $x_1$ denote its first coordinate. We will first show that on the sphere of radius $\dfrac{8\gamma^* \sqrt{d}}{\pi^{3/2}}$ in $\real^d$, more than $3/8$ of the points (under the uniform measure) satisfy
            $$2\gamma^* \leq x_1 \leq \dfrac{8\gamma^* \sqrt{d}}{\pi^{3/2}}.$$

            To prove this, it suffices to show that more than $3/8$ of the points on the sphere $\mS^{d-1}$ of radius 1 satisfy
            $$\dfrac{\pi^{3/2}}{4\sqrt{d}} \leq x_1\leq 1.$$

            Given a subset $B$ of the $d$-dimensional Euclidean sphere with radius 1, let $A_d(B)$ denote its surface area. Define $\phi_d := A_d(\mS^{d-1})$, that is, the total surface area of the $d$-dimensional sphere with unit radius. Let $C$ denote the spherical cap of points satisfying $\epsilon \leq x_1 \leq 1$. Then:

         \begin{align*}
          A_d(C) &= \phi_{d-1}\int_{\epsilon}^1 (1-x_1^2)^{(d-2)/2}dx_1&\\
             &= \phi_{d-1}\bigg(\int_{0}^1 (1-x_1^2)^{(d-2)/2}dx_1& \\
             &\quad - \int_{0}^{\epsilon} (1-x_1^2)^{(d-2)/2}dx_1\bigg)&\\
             &= \dfrac{\phi_d}{2} - \phi_{d-1}\int_{0}^{\epsilon} (1-x_1^2)^{(d-2)/2}dx_1&\\
             &= \dfrac{\phi_d}{2} - \phi_{d-1}\int_{0}^{\epsilon}\bigg( 1-\dfrac{d-2}{2}x_1^2&&\\
             &\quad+\dfrac{(d-2)(d-4)}{8}x_1^4 - \dots &\\
             &\quad + \dfrac{\frac{d-2}{2}\frac{d-4}{2}\dots\frac{d-2i}{2}}{i!}(-x_1^2)^{i}+\dots\bigg)dx_1&\\
             &= \dfrac{\phi_d}{2} - \phi_{d-1}\bigg[x_1-\dfrac{d-2}{6}x_1^3&\\
             &\quad+\dfrac{(d-2)(d-4)}{40}x_1^5 - \dots\bigg]_0^\epsilon .&
         \end{align*}
         
         Note that the binomial expansion above is valid for both even and odd $d$.
         If $\epsilon \leq \sqrt{\frac{2}{d}}$ then
         $$A_d(C) \geq \dfrac{\phi_d}{2} - \epsilon\phi_{d-1}.$$
         Therefore,
         $$\dfrac{A_d(C)}{\phi_d} \geq \dfrac{1}{2} - \epsilon \dfrac{\phi_{d-1}}{\phi_d}.$$
         Standard computations (such as \cite{huber1982gamma}) show that
         $$\phi_{d}=\dfrac{d \pi^{d/2}}{\Gamma (\frac{d}{2}+1)}$$
         where $\Gamma$ is the gamma function. Therefore,
        $$\dfrac{A_d(C)}{\phi_d} \geq \dfrac{1}{2} - \epsilon \dfrac{\Gamma (\frac{d+3}{2})}{\sqrt{\pi} \Gamma (\frac{d+2}{2})}.$$

        By standard properties of the gamma function (see \cite{andrews2000special} for reference),
        $$\Gamma\left(n+\dfrac{1}{2}\right)=\sqrt{\pi}\dfrac{(2n)!}{4^n n!}$$
        where $n$ is any positive integer. Combining this with lower and upper bounds from Stirling's approximation gives us
         \begin{align*}
          \dfrac{A_d(C)}{\phi_d} \geq \dfrac{1}{2} - \epsilon \dfrac{e^2\sqrt{d}}{\sqrt{2}\pi^{3/2}}.
         \end{align*}
         Setting $\epsilon = \dfrac{\pi^{3/2}}{4e^2\sqrt{2d}}$, we find $\dfrac{A_d(C)}{\phi_d}$ is greater than $3/8$.

         Let $p$ denote the probability that a point $x$ drawn uniformly at random on the unit sphere in $d$ dimensions satisfies $x_1 \geq 2\gamma^*$, that is, $x \in C$. By basic properties of the uniform measure on $\mS^{d-1}$ and the above surface area computation, we have
         $$p = \dfrac{A_d(C)}{\phi_d} \geq \frac{3}{8}.$$

         Now, suppose that we draw $N = 16(d-1)$ points on the sphere uniformly at random. By Hoeffding's inequality, the probability that fewer than $3(d-1)$ points of the $N$ points lie on $C$ is at most $\leq e^{-9(d-1)/8}$.
    \end{proof}  

    Let $e_1$ denote the vector $[1,0,0,\ldots, 0]^T \in \real^d$. Without loss of generality, assume that $X_+ = \{0\}$, $X_- = \{2\gamma^*e_1\}$. To ensure non-separability, it is sufficient that $\conv(X_+')$ contains a point $v = \eta e_1$ with $\eta \geq 2\gamma^*$. We use the following result from \cite{10.2307/24490189}:

      \begin{prop}\label{thm:sphereOrigin}
        Suppose we draw $N$ points independently from a spherically symmetric distribution in $\real^d$. Let $p(d,N)$ denote the probability that all $N$ points lie in a common hemisphere. Then
        \begin{equation}\label{eq:sphere_origin}
        p(d,N)=2^{1-N}\sum_{k=0}^{d-1}\binom{N-1}{k}.\end{equation}
      \end{prop}
      
      The right-hand side of \eqref{eq:sphere_origin} is the probability of obtaining fewer than or equal to $d-1$ heads in $N-1$ tosses of a fair coin. Set $N=3d$. Applying Hoeffding's inequality to \eqref{eq:sphere_origin}, we get
      \begin{align*}
      p(d-1,3d) & \leq \exp\left({-2(3d-1)\left(\dfrac{d+1}{2(3d-1)}\right)^2}\right)\\
      &\leq e^{-(d+1)/6}.\end{align*}
      Thus, the probability that the convex hull of $3d$ points drawn uniformly at random from a spherically symmetric distribution contains is at least $1-e^{-(d+1)/6}$.

      Let $V$ be the subspace of $\real^d$ orthogonal to $e_1$ and let $\pi_V$ denote the orthogonal projection on to $V$. Suppose $x$ is drawn uniformly at random from a sphere in $d$-dimensions. Then $\pi_V(x)$ is drawn from a spherically symmetric distribution in $\real^{d-1}$ centered at the origin.

      We have shown that with probability at least $1-e^{-9(d-1)/8}$, at least $3d$ from $X_+'$ satisfy $\langle x,e_1\rangle \geq 2 \gamma^*$. Let $A \subseteq X_+'$ denote the set of these $3d$ points, and let $K = \conv(A)$. Since each $\pi_V(x)$ has spherically symmetric distribution about $0$, by Proposition \ref{thm:sphereOrigin}, with probability at least $1-e^{-d/6}$, $\conv(\pi_V(A)) = \pi_V(K)$ contains the origin in $\real^{d-1}$.

      Since each $x \in A$ satisfies $\pi_1(x) \geq 2\gamma^*$ and $K$ is their convex hull, this implies $K$ contains some point of the form $\eta e_1$ where $\eta \geq 2\gamma^*$. Since $0, \eta e_1 \in X_+^{\aug}$ and $2\gamma^*e_1 \in X_-^{\aug}$, we find
      \begin{align*}
      \PP\left(\mH(S^{\aug}) = \emptyset\right) &\geq 1 - e^{-9(d-1)/8} - e^{-d/6}\\
      &\geq 1 - 2e^{-d/6}.\end{align*}
      The last inequality above is true for $d\geq 2$. For $d=1$, each point in $X_+'$ takes on the values $\pm r$ with equal probabilities. If any point in $X_+'$ equals $r$, then linear separability is violated. The desired result follows.

  \subsection{Proof of Theorem \ref{thm:lin_sep}}\label{proof:lin_sep}

    \begin{proof}
      It suffices to show this for $r = \beta^{-1/2}\sqrt{d/\log(N)}\gamma^*$, as taking smaller $r$ only increases the chance of being linearly separable. Let $H^*$ be the max-margin hyperplane defined by $(w^*,b^*) \in \real^d\times \real$, with $\|w^*\|_2 = 1$. Let $(x + z,y) \in S'$ with $(x,y) \in S$ and $z$ sampled uniformly at random on the sphere of radius $r$. By the definition of max-margin, $y(\langle w^*,x\rangle + b^*) \geq \gamma^*$. By Lemma \ref{lem:spher_cap}, we have
      \begin{align*}
      \PP\left(y(\langle w^*,z\rangle) \geq -\gamma^*\right) &\leq e^{-d(\gamma^*)^2/2r^2}\\
      &= e^{-\beta\log(N)}\\
      &= N^{-\beta}.\end{align*}
      Therefore, $x$ and $x+z$ lie on the same side of $H^*$ with this probability. Taking a union bound over all $N$ points in $S'$, we find that with probability at least $1-N^{1-\beta}$, $H^*$ linearly separates $S^{\aug}$.
    \end{proof}

\subsection{Proof of Lemma \ref{lma:sphere_inclusion}}

    \begin{proof}
      It suffices to prove the result for $r=1$. The proof will proceed similarly to the proof of Lemma 3.1 in \cite{alonso2008isotropy}. With probability 1, the facets of $K$ are simplices. Suppose $\mB_\epsilon(0)\not\subseteq K$. Then there exists at least one facet of $K$ which is contained in a hyperplane orthogonal to some $\theta\in\mS^{d-1}$ such that $\langle z_i, \theta\rangle<\epsilon$ for all $i$. Let $\mu_r$ denote the uniform measure on the sphere $r\mS^{d-1}$. It follows that
        \begin{align*}
            & \PP(\mB_\epsilon(0)\not\subseteq K)\\
            & \leq {N\choose d}\mu_r\left(\left\{x\in r \mathbb{S}^{d-1}: \langle x, \theta\rangle<\epsilon\right\}\right)^{N-d}.
        \end{align*}
        Let $\omega_d$ denote the volume of the $d-$dimensional Euclidean ball. There is some constant $c > 0$ such that if $\frac{c}{\sqrt{d}}<\epsilon<\frac{1}{4}$, then
        \begin{align*}
            \begin{split}
            &\mu_r\left(\left\{ x \in S ^ { d - 1 } :  \langle x , \theta \rangle  > \epsilon \right\}\right) \\
            & =\frac { ( d - 1 ) \omega _ { d - 1 } } { d \omega _ { d } } \int _ { \epsilon } ^ { 1 } \left( 1 - x ^ { 2 } \right) ^ { \frac { d - 3 } { 2 } } d x \\
            & \geq   \frac { ( d - 1 ) \omega _ { d - 1 } } { d \omega _ { d } } \int _ { \epsilon } ^ { 2 \epsilon } \left( 1 - x ^ { 2 } \right) ^ { \frac { d - 3 } { 2 } } d x \\
            & \geq  \frac { ( d - 1 ) \omega _ { d - 1 } } { d \omega _ { d } } \epsilon \left( 1 - 4 \epsilon ^ { 2 } \right) ^ { \frac { d - 3 } { 2 } } \\
            &= c ^ { \prime } \left( 1 - 4 \epsilon ^ { 2 } \right) ^ { \frac { d - 3 } { 2 } } \\
            &= c ^ { \prime } e ^ { \frac { d - 3 } { 2 } \log \left( 1 - 4 \epsilon ^ { 2 } \right) } \\
            &\geq c ^ { \prime } e ^ { - 4 \epsilon ^ { 2 } d }.
            \end{split}
        \end{align*}
        Here $c'$ is some positive constant. Therefore,
        \begin{align*}
            \begin{split}
            & \mathbb { P } \left\{ \mB_\epsilon(0) \nsubseteq K \right\}\\ 
            & \leq \left( \begin{array} { c } { N } \\ { d } \end{array} \right) \left( 1 - c ^ { \prime } e ^ { - 4 \epsilon ^ { 2 } d } \right) ^ { N - d } \\
            & \leq  \left( \frac { e N } { d } \right) ^ { d } \exp \left( ( N - d ) \log \left( 1 - c ^ { \prime } e ^ { - 4 \epsilon ^ { 2 } d } \right) \right) \\
            & \leq \left( \frac { e N } { d } \right) ^ { d } \exp \left( - c ^ { \prime } ( N - d ) e ^ { - 4 \epsilon ^ { 2 } d } \right).
            \end{split}
        \end{align*}

        Let $N = td$. Setting $\epsilon = \frac{1}{2\sqrt{2}}\sqrt{\frac{\log(N/d)}{d}}$, we then get that for $t$ sufficiently large, there is some constant $c''> 0$ such that
        \begin{align*}
            & \PP\left(\mB_\epsilon(0) \nsubseteq K\right)\\
            & \leq e^dt^d\exp\left(-c'(t-1)de^{-\frac{1}{2}\log(t)}\right)\\
            & \leq e^dt^d\exp\left(-c''\sqrt{t}\right)\\
            &= \exp\left(d(1+\log(t)-c''\sqrt{t})\right).
        \end{align*}

        There is some constant $C$ such that if $t \geq C$, then $1+\log(t)-c''\sqrt{t} \leq -1$. Therefore, if $N \geq Cd$, then
        $$\PP\left(\mB_\epsilon(0) \nsubseteq K\right) \leq e^{-d}.$$
    \end{proof}

\subsection{Proof of Theorem \ref{thm:lower_bound_margin_general}}\label{proof:lower_bound_margin}

    \begin{proof}
        Recall that $S' = \{(x_i+z_i^{(j)},y)\}_{i \in [n],j \in [N]}$ where each $z_i^{(j)}$ is drawn uniformly at random from the sphere of radius $r$ and $S = \{(x_i,y)\}_{i \in [n]}$. For any $i \in [n]$, let $A_i := \{x_i + z_i^{(j)}\}_{j \in [N]}$ and $K_i := \conv(A_i)$. By Lemma \ref{lma:sphere_inclusion},with probability at least $1-e^{-d}$, $\mB_\rho(x_i) \subseteq K_i$ where
        $$\rho = \frac{1}{2\sqrt{2}}\sqrt{\frac{\log(N/d)}{d}}r.$$
        By the union bound, this holds for all $x_i$ with probability at least $1-ne^{-d}$.

        Suppose that $\mB_\rho(x_i) \subseteq K_i$ for all $i \in [n]$, and suppose $(w,b)$ is a linear separator of $S^{\aug}$. Let $H = \{x | \langle w,x \rangle + b \geq 0\}$, and let $H^+ = \{x | \langle w,x \rangle + b\geq 0\}$. Therefore, for all $x \in X_+^{\aug}$, $x \in H^+$. Fix some $i \in [n]$ such that $(x_i, y_i) \in S$ has label $y_i = 1$. Since $A_i \subseteq H^+$, by convexity we find that $K_i := \conv(A_i)$ satisfies $K_i \subseteq H^+$. Since $\mB_\rho(x_i) \subseteq K_i$, we find $\mB_\rho(x_i) \subseteq H^+$. Since $H = \partial H^+$, we find that $d(x,H) \geq d(x, \partial \mB_\rho(x_i)) = \rho$.

        Using an analogous argument for $X_-$ and applying the union bound, we find that with probability at least $1-ne^{-d}$, any linear separator of $S^{\aug}$ must have margin at least $\rho$ with respect to $X_+, X_-$.

        It now suffices to show that there exists a linear classifier. We will use method to the proof of Theorem \ref{thm:lin_sep}. The only difference is that when taking a union bound over all the perturbations, we have $|S'| = nN$ instead of $|S'| = 2N$. Thus, with probability at least $1-nN^{1-\beta}$, $S^{\aug}$ is linearly separable. Taking a union bound gives the desired result.
    \end{proof}

\section{Proof of Results in Section \ref{sec:nonlinear_classifiers}}\label{sec:append_nonlinear}

    	We first prove a general proposition regarding some basic properties of $\epsilon$-respectful functions.

	\begin{prop}\label{prop:local_convex}
		Let $\epsilon \in [0,\infty]$.
		\begin{enumerate}
			\item $\mR_0(S) = \mR(S)$.
			\item If $\epsilon \leq \epsilon'$, then $\mR_{\epsilon'}(S) \subseteq \mR_{\epsilon}(S)$.
			\item If $S$ is bounded, then $\exists M \in\real_{\geq 0}$ such that $\forall \epsilon \geq M$, $\mR_\epsilon(S) = \mR_\infty(S)$.
			\item For any set $A \subseteq \real^d$, define
			$$A[\epsilon] = \bigcup_{B \subseteq A,~R(B) \leq \epsilon}\conv(B).$$
			Then $\mR_\epsilon(S) \neq \emptyset$ iff $X_+[\epsilon] \cap X_-[\epsilon] = \emptyset$.
		\end{enumerate}
	\end{prop}

	\begin{proof}
		(1): Note that $R(A) = 0$ iff $|A| \leq 1$. Thus, $f \in \mR_\epsilon(S)$ iff $\forall (x,y) \in S, f(x) = y$.

		(2): Suppose $f \in \mR_{\epsilon'}(S)$ and $\epsilon\leq \epsilon'$. If $A \subseteq X_+$ and $R(A) \leq \epsilon$, then $R(A) \leq \epsilon'$, so for all $x \in \conv(A), f(x) = 1$. By symmetry for $X_-$, we find $f \in \mR_\epsilon(S)$.

		(3): Since $S$ is bounded, there is some finite $M$ such that $R(X_+), R(X_-) \leq M$. Suppose $f \in \mR_{M}(S)$. Then for any $A \subseteq X_+$, we have $R(A) \leq R(X_+) \leq M$ and $f(x) = 1$ on $\conv(X_+)$, so $f(x) = 1$ for all $x \in \conv(A)$. By symmetry for $X_-$, we find that if $f \in \mR_M(S), f \in \mR_\infty(S)$. By (2), we are done.

		(4): By definition, we know that if $f \in \mR_\epsilon(S)$, then for all $x_1 \in X_+[\epsilon]$, $f(x_1) =1 $ and for all $x_2 \in X_-[\epsilon], f(x_2) = -1$, so $X_+[\epsilon]\cap X_-[\epsilon] = \emptyset$. Conversely, if these two sets are disjoint, we can define $f(x)$ to be $1$ on $X_+[\epsilon]$ and $-1$ elsewhere. If $A \subseteq X_+$ and $R(A) \leq \epsilon$, then $\conv(A) \subseteq X_+[\epsilon]$ and so $f(x) = 1$ for all $x \in \conv(A)$. By symmetry for $X_-$, we conclude that $f \in \mR_\epsilon(S)$.
	\end{proof}

	We will also use the following bound on $\gamma_f(S)$ in terms of the distance $d(X_+,X_-)$ between $X_+$ and $X_-$.

	\begin{lemma}\label{lem:Delta}
		For $f:\real^d \to \{\pm 1\}$, $\gamma_f(S) \leq \dfrac{d(X_+,X_-)}{2}$.
	\end{lemma}

	\begin{proof}
		If $f \notin \mR(S)$, then $\gamma_f(S) = -\infty$. Otherwise, fix $\delta > 0$. Then $\exists x_1 \in X_+, x_2 \in X_-$ such that
		$$d(x_1,x_2) < d(X_+,X_-)+\delta.$$
		Consider the point $z = (x_1+x_2)/2$. By construction,
		$$d(x_1,z) = d(x_2,z) < \frac{d(X_+,X_-)+\delta}{2}.$$
		Let $y = f(z)$. Then for either $j= 1$ or $j = 2$, we have $d(x_j,f^{-1}(-y_j)) < \frac{d(X_+,X_-)+\delta}{2}$.
		Thus, for all $\delta > 0$, $\exists (x,y) \in S$ such that
		$$d(x,f^{-1}(-y)) < \dfrac{d(X_+,X_-)+\delta}{2}.$$
		Since this holds for all $\delta > 0$, $\gamma_f(S) \leq \dfrac{d(X_+,X_-)}{2}$.
	\end{proof}

	\subsection{Proof of Theorem \ref{thm:nonlinear_lower_bound}}

		\begin{proof}
			If $\mR_\epsilon(S^{\aug}) = \emptyset$, then we are done. Otherwise, consider the following sets:
			$$X_+'' = \bigcup_{A \subseteq X_+\cup X_+',~R(A) \leq \epsilon} \conv(A)$$
			$$X_-'' = \bigcup_{B \subseteq X_-\cup X_-',~R(B) \leq \epsilon} \conv(B)$$

			Suppose $|X_+'| \leq d$. By Proposition \ref{prop:local_convex}, $X_+'' \cap X_-'' = \emptyset$. Therefore, we can define the function $f$ that is $1$ on $X_+''$ and $-1$ elsewhere. By construction, $f \in \mR_\epsilon(S\cup S')$. Let $V = (\partial X_+'')\cap(X_+\cup X_+')$. That is, $V$ is the set of points on the boundary of $X_+''$ that are also in $X_+\cup X_+'$. It suffices to show that $X_+ \cap V \neq \emptyset$.

			Suppose $\exists x \in X_+$ such that $x \notin V$. Then $x \in \intr(\conv(V))$, which implies $|V| > d$. Since $|X_+'| \leq d$, $\exists x \in X_+\cap V$. Since $x \in \partial X_+''$, $f$ has 0 margin on $X_+$. If $|X_-'| \leq d$ we can define an analogous function using $X_-''$ to achieve 0 margin at some point in $X_-$.
		\end{proof}		

	\subsection{Proof of Theorem \ref{thm:nonlinear_lower_bound_2}}

		\begin{proof}
			Let $X_+, X_-$ be disjoint sets of size $n, m$ such that for any $x_1, x_2 \in X_+\cup X_-$ with $x_1 \neq x_2$, $d(x_1,x_2) > \epsilon+2r$, and define
			$$S = \bigg(X_+\times\{1\}\bigg) \bigcup \bigg(X_-\times\{-1\}\bigg).$$

			Suppose $S' \subseteq S_r$. Therefore, $X_{\pm}' \subseteq (X_{\pm})_r$. Thus, for any $a,b \in X_+$, either there is some $x \in X_+$ such that $d(a,x), d(b,x) \leq r$, or $d(a,b) > \epsilon$. In particular, if $A \subseteq X_+\cup X_+'$ and $R(A) \leq \epsilon$, then $A \subseteq \mB_r(x)$ for some $x \in X_+$.

			Define the sets
			$$X_+'' = \bigcup_{A \subseteq X_+\cup X_+',~R(A) \leq \epsilon} \conv(A)$$
			$$X_-'' = \bigcup_{B \subseteq X_-\cup X_-',~R(B) \leq \epsilon} \conv(B)$$
			By construction, $X_+'' \cap X_-'' = \emptyset$. Define $f$ to be $1$ on $X_+''$ and $-1$ elsewhere. Let $f \in \mR_\epsilon(S\cup S')$. Suppose that for all $x \in X_+$, $d(x,f^{-1}(-1)) > 0$. Then $x$ must be in the interior of $X_+''$. By the argument above, this implies that there is some set $A \subseteq X_+' \cap \mB_r(x)$ such that $x \in \intr(\conv(A))$. Hence, $|X_+' \cap \mB_r(x)| \geq d+1$. Thus, $|X_+'| \geq (d+1)|X_+|$. An analogous argument shows that to guarantee positive worst-case margin, we must also have $|X_-'| \geq (d+1)|X_-|$.
		\end{proof}	

	\subsection{Proof of Lemma \ref{lem:gamma_epsilon}}

		\begin{proof}
			Let $X_+ = \{a\}, X_- = \{b\}$ for points $a, b$ satisfying $d(a,b) \geq 3r$. It suffices to consider the case where $X_+' = \mB_r(a)$, $X_- = \mB_r(b)$. We can then define $f$ as follows. For $x \in \mB_r(a)$, $f(x) = 1$ if $x = a$ or $d(x,a) > r-\epsilon$, and $-1$ otherwise. By construction, if $A \subseteq X_+\cup X_+'$ and $R(A) \leq \epsilon$, then either $A = \{x\}$ or $f(x) = 1$ on $A$. Moreover, $d(a,f^{-1}(-1)) = 0$. We can define $f$ analogously on $\mB_r(b)$. In either case, $f \in \mR_\epsilon(S\cup S')$ but $\gamma_f(S,S') = 0$.
		\end{proof}	

	\subsection{Proof of Theorem \ref{thm:suff_nonlinear}}

		\begin{proof}
			We form $X_+'$ by selecting points forming a $d$-simplex $C$ of circumradius $R(C) \leq \epsilon$ about each point in $X_+$. Note that this requires exactly $|X_+|(d+1)$ points. We do the same for $X_-'$. Since each $d$-simplex $C \subseteq X_+'$ has circumradius $R(C) \leq \epsilon$, we are guaranteed that if $f \in \mR_\epsilon(S\cup S')$, then $f(C) = 1$ on $\conv(C)$. For the point $x \in X_+ \cap \intr(\conv(C))$, we are guaranteed that $d(x,f^{-1}(-1)) > 0$. Thus, $\gamma_f(S,S') > 0$.
		\end{proof}	

	\subsection{Proof of Theorems \ref{thm:nonlinear1} and \ref{thm:nonlinear2}}\label{proof:nonlinears}

		\begin{proof}
			We will use similar techniques to the proof of Theorem \ref{thm:lower_bound_margin_general}. Suppose $\epsilon \in (0,\infty]$ and $r \in (0,\infty)$ satisfies $r \leq \epsilon$. Recall that $S' = \{x_i+z_i^{(j)}, y\}_{i \in [n],j \in [N]}$ where each $z_i^{(j)}$ is drawn uniformly at random from the sphere of radius $r$.

			Fix $i$. Define $A_i = \{x_i+z_i^{(j)}\}_{j \in [N]}$, and $K_i = \conv(A_i)$. Then by Lemma \ref{lma:sphere_inclusion}, we know that with probability at least $1-e^{-d}$, $\mB_\rho(x_i) \subseteq K_i$ where
			$$\rho = \frac{1}{2\sqrt{2}}\sqrt{\frac{\log(N/d)}{d}}r.$$

			Since $\|z_i^{(j)}\|_2 = r \leq \epsilon$, we have $A_i \subseteq \mB_\epsilon(x_i)$. Hence, $R(A_i) \leq \epsilon$. By $\epsilon$-respectfulness, we know that if $f \in \mR_\epsilon(S\cup S')$, then for all $x \in K_i$, $f(x) = y_i$. In particular, $f(x) = y_i$ for all $x \in \mB_{\rho}(x_i)$. This then implies that $d(x_i,f^{-1}(-y_i)) \geq \rho$.

			Taking a union bound over all $i \in [n]$, we find that with probability at least $1-ne^{-d}$, if $f \in \mR_\epsilon(S,S'$, then) for all $i \in [n]$, $d(x_i,f^{-1}(-y_i)) \geq \rho$. In particular, this implies that $\gamma_f(S,S') \geq \rho$.

			By Proposition \ref{prop:local_convex}, $\mR_\infty(S\cup S') \neq \emptyset$ iff $\conv(X_+^{\aug}) \cap \conv(X_-^{\aug}) = \emptyset$. By the separating hyperplane theorem, this occurs iff there is some separating hyperplane with positive margin. Applying Theorem \ref{thm:lin_sep} and the union bound as in the proof of Theorem \ref{thm:lower_bound_margin_general}, we prove Theorem \ref{thm:nonlinear1}.

			The first part of Theorem \ref{thm:nonlinear2} was proved above. The second will follow from the following lemma, which we prove in the next section.

			\begin{lemma}\label{lem:nonempty_overall}Suppose $\mR_\epsilon(S) \neq \emptyset$. If
			$$r \leq \epsilon < \frac{d(X_+, X_-)}{4}$$
			then for any $S' \subseteq S_r$, $\mR_\epsilon(S\cup S')$ is non-empty.\end{lemma}   

        \end{proof}

    \subsection{Proof of Lemma \ref{lem:nonempty_overall}}

 		We first prove the following lemma.

		\begin{lemma}\label{lem:local_convex_nonempty}
			Suppose there is $f \in \mR_\epsilon(S)$ such that $\gamma_f(S) = \phi$. If $r < \phi-\epsilon$ and $S' \subseteq S_r$, then $f \in \mR_\epsilon(S\cup S')$.
		\end{lemma}

		\begin{proof}
			Let $f \in \mR_\epsilon(S)$ satisfy $\gamma_f(S) = \phi$. We first prove that $f \in \mR_\epsilon(S\cup S')$. Suppose $A \subseteq X_+\cup X_+'$ with $R(A) \leq \epsilon$, and let $x \in A$. This implies that $d(x,X_+\cup X_+') \leq \epsilon$.

			Fix $\delta > 0$. Then, $\exists u \in X_+ \cup X_+'$ such that $d(x,u) \leq \epsilon + \delta$. Since $S' \subseteq S_r$, $X_+'\cup X_+ \subseteq (X_+)_r$. Therefore, $\exists v \in X_+$ such that $d(v,u) \leq r$. We then have
			$$d(x,X_+) \leq d(x,u) \leq d(x,v) + d(v,u) \leq r+\epsilon + \delta.$$

			Since this holds for all $\delta > 0$ and $r+\epsilon < \phi$, this implies $d(x,X_+) < \phi$. Since $\gamma_f(S) = \phi$, this implies $f(x) = 1$. Performing an analogous argument for $X_-$ shows that $f \in \mR_\epsilon(S\cup S')$.
		\end{proof}

		Suppose $\mR_\epsilon(S) \neq \emptyset$. Define
		$$\gamma^*(S) = \sup_{f \in \mR_\epsilon(S)} \gamma_f(S).$$
		This is the maximum margin of any $\epsilon$-respectful classifier on $S$. Lemma \ref{lem:local_convex_nonempty} implies that if $r \leq \epsilon < \frac{\gamma^*(S)}{2}$, and $S' \subseteq S_r$, then $\mR_\epsilon(S\cup S') \neq \emptyset$.

		We would like to guarantee that for certain $\epsilon$, $\theta$ is not too small. Recall that by Lemma \ref{lem:Delta}, for any $f \in \mR_\epsilon(S)$,
		$$\gamma_f(S) \leq \frac{d(X_+,X_-)}{2}.$$
		We will show that the converse holds when $\epsilon$ is sufficiently small.

		\begin{lemma}\label{lem:epsilon_Delta}
			For all $\epsilon < \frac{d(X_+,X_-)}{2}$, there is some $f \in \mR_\epsilon$ such that $\gamma_f(S) = \frac{d(X_+,X_-)}{2}$.
		\end{lemma}

		\begin{proof}
			By Example \ref{ex:nn}, for $\epsilon < \frac{d(X_+,X_-)}{2}$, $f_{NN} \in \mR_\epsilon(S)$. For $x \in S$ and any point $z$ with $d(x,z) < \frac{d(X_+, X_-)}{2}$, $f(z) = f(x)$. Therefore, $\gamma_{f_{NN}}(S) = \frac{d(X_+, X_-)}{2}$.
		\end{proof}

		Combining Lemmas \ref{lem:local_convex_nonempty} and \ref{lem:epsilon_Delta}, we complete the proof of Lemma \ref{lem:nonempty_overall}.

\subsection{Proof of Theorem \ref{thm:nonlinear3}}
\label{proof:upper_bound_margin}

	We will use the following result, proved in \cite{alonso2008isotropy}, as a part of the proof for Theorem 3.1 in their paper.

    \begin{prop}\label{prop:inradius}
		There exist absolute constants $C_1$ and $C_2$ such that if $\{z_i\}^N_{i=1}$ are independent random vectors on $\mS^{d-1}$, $N>d$, and $K' = \conv\{\pm z_1,\dots, \pm z_N\}$, then
		$$\mathbb{P}\left(\sup_{i=1,\dots,l} \dfrac{1}{\vol(F_i)} \int_{F_i} |x|^2dx \leq C_1 \dfrac{\log(2N/d)}{d} \right) $$
		$$\geq 1 - 2e^{-C_2 d \log(N/d)}$$
		where $\{F_1,\dots,F_l\}$ is the set of facets of $K'$.
    \end{prop}

		Above, $\vol(F_i)$ refers to the $k$-dimensional volume where $k$ is the dimension of $F_i$. Note that \cite{Klartag2009} gives a similar result to Proposition \ref{prop:inradius} in Corollary 2.4, but for standard Gaussian vectors instead of random points on the unit sphere. Equipped with this, we can proceed.

	\begin{proof}[Proof of Theorem \ref{thm:nonlinear3}]

		Without loss of generality, we may assume $X_+ = \{0\}$ and $X_+' = \{z_i\}_{i \in [N]}$. Let $K = \conv(\{0\} \cup X_+)= \conv(\{0\} \cup \{z_i\}_{i\in [N]})$ and let $K' = \conv(\{0\} \cup X_+^{\aug})= \conv(\{0\} \cup \{\pm z_i\}_{i \in [N]})$. Note that $K \subseteq K'$. Applying Proposition \ref{prop:inradius} to $K'$ (where we scale up by $r$) and applying Jensen's inequality, we have that there are some constants $C_1, C_2$ such that
        $$\mathbb{P}\left(\forall i\in [l],  \dfrac{1}{\vol(F_i)} \int_{F_i} |x|dx \leq \sqrt{C_1 \dfrac{\log(2N/d)}{d}}r \right)$$
        $$\geq 1 - 2e^{-C_2 d \log(N/d)}.$$
        Here, $\{F_i\}_{i \in [l]}$ are the facets of $K'$. Define
        $$\delta = \sqrt{C_1\dfrac{\log(2N/d)}{d}}r.$$

        The term $\vol(F_i)^{-1} \int_{F_i} |x|dx$ is the average distance of the points on the facet $F_i$ to the origin. If this average is bounded above by $\delta$, then there is at least on point on the facet that is at a distance less than or equal to $\delta$ to the origin. Therefore, with probability at least $1 - 2\exp(-C_2 d \log(N/d))$, there is a point on each $F_i$ of distance at most $\delta$ to the origin. Thus, $d(0,\partial K') \leq \delta$. Since $K \subseteq K'$, this implies that with the same probability, $d(0, \partial K) \leq \delta$.

        Now, define a function $f: \real^d \to \real$ to be $+1$ on $K$ and $-1$ elsewhere. Since $R(X_+'), R(X_-') \leq r \leq \epsilon$, Proposition \ref{prop:local_convex} implies that $\conv(X_+') \cap \conv(X_-') = \emptyset$. Therefore $f$ is well-defined and $f \in \mF_\epsilon(S^{\aug})$. We then have by construction of $f$,
        \begin{align*}
        \gamma_f(S,S') \leq d(0, f^{-1}(-1)) = d(0,\partial K).\end{align*}
        Therefore, with probability at least $1 - 2\exp(-C_2 d \log(N/d))$,
        $$\gamma_f(S,S') \leq \sqrt{C_1\dfrac{\log(2N/d)}{d}}r.$$

    \end{proof}

\end{document}